\documentclass[conference]{IEEEtran}

\IEEEoverridecommandlockouts

\usepackage{cite}
\usepackage{amsmath,amssymb,amsfonts}
\usepackage{graphicx}
\usepackage{textcomp}
\usepackage{xcolor}
\usepackage{algorithm}
\usepackage{algorithmic}
\usepackage{amsthm}
\usepackage{comment}
\usepackage{ifthen}
\usepackage{graphicx}
\usepackage{subfig}

\def\BibTeX{{\rm B\kern-.05em{\sc i\kern-.025em b}\kern-.08em
		T\kern-.1667em\lower.7ex\hbox{E}\kern-.125emX}}

\newtheorem{lemma}{Lemma}

\newtheorem{corollary}{Corollary}
\newtheorem{remark}{Remark}
\newtheorem{definition}{Definition}

\newtheorem{theorem}{Theorem}

\newtheorem{proposition}{Proposition}

\newcommand{\verts}{S}
\newcommand{\edges}{\mathcal{E}}

\newcommand{\neighbor}{N}
\newcommand{\rad}{\text{rad}}
\newcommand{\pr}{\text{Pr}}

\newcommand{\gb}{\mathcal{GB}}

\DeclareMathOperator*{\argmax}{arg\,max}
\allowdisplaybreaks

\newcommand{\modified}[1]{{\color{black}{#1}}}

\newcommand\ucbcomparewidth{0.3\linewidth}

\newcommand\benchmarkwidth{0.3\linewidth}

\newboolean{short_paper}
\setboolean{short_paper}{false}

\newcommand{\ifshortpaperthenelse}[2]{\ifthenelse{\boolean{short_paper}}{#1}{#2}}

\usepackage{setspace}
\title{Multi-armed Bandit Learning on a Graph}

\author{Tianpeng Zhang$^{\dag1}$ \quad Kasper Johansson$^{\dag2}$ \quad Na Li$^{1}$
	\thanks{$^\dag$ T. Zhang and K. Johansson  contributed to this work equally.}
\thanks{$^{1}$ T. Zhang and N. Li are with Harvard School of Engineering and Applied Sciences, Cambridge, MA, USA. Emails: tzhang@g.harvard.edu, nali@seas.harvard.edu.  Zhang and Li are supported by ONR YIP N00014-19-1-2217, NSF AI institute 2112085, and NSF CNS 2003111. }
\thanks{$^{2}$ K. Johansson is from the Department of Electrical Engineering at Stanford University, Stanford, CA, USA. The work was done during K. Johansson's visit to Harvard for his master thesis project at KTH, Sweden, supported by the Lars Magnus Ericsson Research Foundation and the Rudolph Carl Norberg Foundation. Email: kasperjo@stanford.edu. }
}

\begin{document}

\maketitle

\begin{abstract}
	
	The multi-armed bandit(MAB) problem is a simple yet powerful framework that has been extensively studied in the context of decision-making under uncertainty. In many real-world applications, such as robotic applications, selecting an arm corresponds to a physical action that constrains the choices of the next available arms (actions). Motivated by this, we study an extension of MAB called the graph bandit, where an agent travels over a graph to maximize the reward collected from different nodes. The graph defines the agent's freedom in selecting the next available nodes at each step. We assume the graph structure is fully available, but the reward distributions are unknown. Built upon an offline graph-based planning algorithm and the principle of optimism, we design a learning algorithm, \texttt{G-UCB}, that balances long-term exploration-exploitation using the principle of optimism. We show that our proposed algorithm achieves $O(\sqrt{|S|T\log(T)}+D|S|\log T)$ learning regret, where $|S|$ is the number of nodes and $D$ is the diameter of the graph, which matches the theoretical lower bound $\Omega(\sqrt{|S|T})$ up to logarithmic factors. To our knowledge, this result is among the first tight regret bounds in non-episodic, un-discounted learning problems with known deterministic transitions. Numerical experiments confirm that our algorithm outperforms several benchmarks. 
\end{abstract}

\section{Introduction} \label{sec:intro}

The multi-armed bandit~(MAB) problem is a popular framework for studying decision-making under uncertainty~\cite{slivkins2019introduction,lattimore2020bandit}. The MAB consists of $n$ independent arms in the simplest setting, each providing a random reward from their corresponding probability distributions. Without knowing these distributions, an agent picks one arm for each round and tries to maximize the total expected reward in $T$ rounds. The agent needs to balance between \textit{exploring} arms to learn the unknown distributions and \textit{exploiting} the current knowledge by selecting the arm that has provided the highest reward. The MAB problem enjoys a wide range of applications, such as digital advertising~\cite{schwartz2017customer}, portfolio selection~\cite{huo2017risk}, optimal design of clinical trials~\cite{villar2015multi}, etc. Many online learning algorithms have been developed and analyzed, such as confidence bound (UCB) methods and Thompson Sampling~\cite{slivkins2019introduction,lattimore2020bandit}.

One limiting assumption of the classical MAB is that the agent has access to \textit{all} the arms at each time and can \textit{freely} switch between different arms, which is not the case for many real-world applications. For instance, when a robot explores an unknown physical environment, the location of the robot will influence the locations that the robot can visit next. Such constraints on the available arms can often be modeled as a graph not captured in the MAB framework. This paper aims to understand how the agent could balance exploration and exploitation under such graph constraints.

Specifically, we consider the problem that an agent traverses an undirected, connected graph. Upon arrival at node $s$, the agent receives a random reward $r_s$ from a probability distribution $P(s)$. The objective is to maximize the expected accumulated reward over $T$ steps. The agent \textit{knows} the graph structure but \textit{not} the reward distributions. We refer to this problem as the \textit{graph bandit} problem. There are many applications of the graph bandit problem, including street cleaning robots trying to maximize the trash-collection efficiency for a neighborhood ~\cite{Rayguru2021
}, a mobile sensor moving between different spots to find the place that receives the strongest signals, or a drone that provides internet access flying over a network of rural/suburban locations to maximize the use of the communication channels\cite{Qiu2019
}. The robot typically has a map (or a graph) of the environment in these applications. Still, each location's quality, i.e., the signal strength in the mobile sensor application or the demand for internet access/communication in the drone application, is often stochastic with unknown distributions.  

The graph bandit problem can be formulated as a reinforcement learning (RL) problem by modeling the underlying process as a Markov Decision Process (MDP) where the nodes are the states, and the next available nodes are the actions. \modified{Nevertheless, our setting is of additional interest because many robotic learning problems can be formulated as graph bandit problems, like those described previously.} \modified{ In these applications, directly applying generic RL algorithms may ignore the structure of the graph bandit problem, which could lead to high computational costs, high sample complexity, and/or high regrets, especially for large-size problems. Therefore, it is worth considering whether any improvement in computational efficiency and learning performance can be achieved by leveraging the graph bandit problem structure.} 

\textit{Contributions.} 
In this paper, we propose a learning algorithm, \texttt{G-UCB}, that adopts the principle of optimism by establishing Upper Confidence Bounds (UCB) for each node and carefully controlling the number of visits of the destination nodes under the planned policy. We note that \texttt{G-UCB} is a \textit{single-trajectory} learning scheme without forcing the agent to return to initial nodes and restarting the process. Compared with state-of-the-art UCB-based generic RL algorithms, like UCRL2\cite{Jaksch2010a}, we make two innovations that improve computational efficiency and learning performance using graph bandit structures: \textbf{i)} We propose a much more efficient offline planning algorithm by formulating offline planning as the shortest path problem (Remark~\ref{rem:computation}). \textbf{ii)} We propose a different UCB definition leading to prominent empirical performance improvement (Remark~\ref{rem:UCB}). 



We establish a minimax regret bound $O(\sqrt{|S|T\log(T)}+D|S|\log(T))$ that matches the theoretical lower bound $\Omega(\sqrt{|S|T})$ up to logarithmic factors (Remark~\ref{rem:MAB}). To our knowledge, this result is among the first tight regret bounds in non-episodic, un-discounted learning problems with known deterministic transitions. We also establish an instance-dependent bound $O(\frac{|S|\log(T)}{\Delta})$, where $\Delta$ is the `gap' in mean rewards between the top two nodes. Our results are tighter than the results in \cite{Jaksch2010a} under the graph bandit setting (Remarks~\ref{rem:RL} and \ref{rem:thm-Delta}). 
\ifshortpaperthenelse{Given the space limit, detailed proofs of the theorems are included in our technical report\cite[Appendices A and B]{OurTechnicalReport}.}{The detailed proofs can be found in Appendices \ref{appendix:thm-G-UCB} and \ref{appendix:thm-UCB-DEST-DELTA}.} 

Lastly, we show in simulations that our proposed algorithm achieves lower learning regret than several benchmarks, including state-of-the-art algorithms like UCRL2\cite{Jaksch2010a} and UCB-H\cite{ISQLPROVABLYEFFICIENT} (Section~\ref{sec:benchmark-on-grid}). We also demonstrate the applicability of our framework in a synthetic robotic internet provision problem (Section~\ref{sec:robot-application}). 
More numerical studies can be found in \ifshortpaperthenelse{our technical report\cite[Appendix F]{OurTechnicalReport}}{Appendix \ref{append:additional-simulations}}.

\subsection{Related work}

\noindent\textbf{MAB.} The graph bandit problem is an extension of MAB to model the case where an arm pulled at one time will constrain the next available arms. 
Note that if the graph is fully connected, the graph bandit problem becomes the classical MAB problem. Recent MAB work has considered the effect among arms pulled at different times. For instance, \cite{
	simchi2019phase,chen2020minimax} assume there is a cost or constraint for switching from one arm to a different arm. We could view this as a soft constraint for switching between arms, while the graph bandit models the hard constraints in switching arms. 

\vspace{5pt}

\noindent\textbf{Reinforcement Learning.} 
The graph bandit problem considered in this paper can be formulated as a \textit{non-episodic, un-discounted} RL problem with deterministic state-transition dynamics but random rewards. There has been a quickly growing literature in designing efficient RL algorithms, especially for the \textit{episodic} or \textit{infinite discounted} MDPs, e.g., \cite{
	NIPS2013_68053af2,
	ISQLPROVABLYEFFICIENT,
	he2021nearly,
	Azar2017
}. 
In contrast, we focus on the non-episodic, un-discounted MDPs. Under this setting, \cite{Jaksch2010a, agrawal2017optimistic} develop RL algorithms with near-optimal regret bounds, but our lower regret bound is lower, especially concerning the problem size, since we exploit the known transitions in graph bandit.  \modified{The line of work by \cite{neu2010online} and \cite{NIPS2013_68053af2} also consider RL problems under known transitions, but these studies are under the episodic setting, which is different from us. We are aware of only one work assuming both known transition and the non-episodic setting\cite{NIPS2010_7bb06076}.
	However, they assume the reward distribution changes adversarially in time, which is more complex than our setting. Their $O(S^{1/3}T^{2/3})$ regret is understandably worse than our result due to the additional complexity. }

\vspace{5pt}

\noindent\textbf{Stochastic shortest path.} 
Our problem is also closely related to the stochastic shortest path (SSP) problems~\cite{bertsekas2015dynamic}. Indeed, our offline planning algorithm is built upon a deterministic shortest-path problem. However, SSP work mainly considers the episodic setting~\cite{rosenberg2020near,tarbouriech2021stochastic}. Moreover, our graph bandit problem has unique features (deterministic node transitions but random rewards) that could facilitate the online learning algorithm design, making it different from SSP problems.

\vspace{5pt}
\modified{
	
	\noindent\textbf{Other related work.}
	The Canadian Traveller Problem could be an interesting extension to our work\cite{
		nikolova2008route}, where the graph is not fully known in advance but is gradually discovered as the agent travels. UCT\cite{coquelin2007bandit} and AMS\cite{chang2005adaptive} also use optimism to guide exploration-exploitation. They can be viewed as special cases of graph bandit, where the graph is a tree. One crucial difference is that we consider the infinite horizon undiscounted reward, while UCT/AMS consider the finite horizon or discounted reward. These studies are rather different from this paper but are good inspirations for future work.
}

\vspace{5pt}
\textit{Notation:} Let $x_{l:h}$ denote the ordered sequence of entities $(x_l,x_{l+1},...,x_{h})$. If $x_{l:h}$ are real numbers, $x_{l:h}\leqq x$~(or $\geq$) means $x_i\leqq x$~(or $\geq$) for all $i = l,l+1,...,h$. For $n\leq m$, $[n:m]$ denotes the set of integers $\{i:n\leq i\leq m\}$, and for $m\geq 0$, $[m]=[0:m]$. Both $[\mathbf{p,q}]$ and $(\mathbf{p,q})$ denote the concatenation of two finite sequences $\mathbf{p}$ and $\mathbf{q}$. Given a set $S$, $|S|$ denotes its carnality. For two finite sets $V$ and $U$, let $V\times U =\{(u,v): u\in U, v\in V\}$ denote the Cartesian product between $U$ and $V$.  

\section{Problem formulation}\label{sec:problem-formualtion}
We consider an agent traveling over an undirected, connected graph, denoted by $G=(S,\edges)$ with vertices $S=\{1,2,...,n\}$ and edges $\edges\subseteq S\times S$. For each $s\in S$, denote $\neighbor_s=\{v\in S:(s,v)\in \edges\}\cup\{s\}$ as its neighboring nodes. Note that $s\in \neighbor_s$ for all $s\in S$.
Each $s\in S$ is associated with a reward distribution $P(s)$. Whenever the agent visits a node $s$, it receives a random reward, $r_s$, sampled independently from $P(s)$.  
We assume all rewards are bounded in $[0,1]$. Let $\mu_s$ be the expected reward at node $s$, i.e., $\mu_s=\mathbb{E}[r_s]$ and denote $\mu^*=\max_s \mu_s$ the highest expected reward in the graph. Without loss of generality, we assume that there is a unique node $s^*$ such that $\mu_{s^*}=\mu^*$. Given a path of nodes $s_{0:T}=(s_0,s_1,...,s_{T})$ with $T\geq 1$, we say that the path is \textit{admissible} if $(s_t,s_{t+1})\in\mathcal{E}$ for all $t=0,1,...T-1$ and define its \textit{length} as $T$. If a path contains only one node, the path is also admissible with \textit{length} $0$. For any two nodes $(s,s')$, let $l(s,s')$ be the length of the shortest path connecting them, i.e., $l(s,s'):=\min\{T: s_{0:T} \text{ is an admissible path with } s_0=s, s_T=s'\}.$ We denote the diameter of the graph as 
\begin{equation}\label{eq:diameter}
	D:=\max_{s,s'\in S} l(s,s') .
\end{equation}

We let $V(s_{0:T})$ be the expected cumulative reward for an admissible path $s_{0:T}$.
\begin{equation}
	V(s_{0:T}) :=\mathbb{E}[\sum_{t=0}^T r_{t}]= \sum_{t=0}^T \mu_{s_t},
\end{equation}
where $r_t$ is the reward collected at node $s_t$ at time $t$. 

Motivated by the applications in Section~\ref{sec:intro}, we assume the agent knows the graph $G$ but does not know the reward distributions $P$. The goal of the agent is to travel on the graph to maximize the expected reward $V(s_{0:T})$ over the period $[0:T]$. To do so, the agent must balance the exploration and exploitation while following the graph constraints. At each time $t$, a learning algorithm $\mathcal{A}$ decides the next admissible node that the agent should visit based on the travel and reward histories. To evaluate the performance of the learning algorithm, we define the regret of an algorithm $\mathcal{A}$ as follows:
\begin{equation}\label{def:learning-regret}
	\begin{aligned}
		&R^*(\mathcal{A},s_0,T):= \mathbb{E}[T\mu^* - \sum_{t=1}^T r_{t}],\\
	\end{aligned}
\end{equation}
where $r_t$ is the reward at time step $t$ resulting from executing the algorithm $\mathcal{A}$ and $s_0$ is the initial node. The definition of $R^*(\mathcal{A},s_0,T)$ follows the standard definition of regret in the un-discounted, infinite horizon RL literature\cite{Jaksch2010a,agrawal2017optimistic}. It measures the difference between the highest attainable expected reward, $T\mu^*$, and the expected accumulated reward of the executed path.
We want to highlight that we consider the online learning setting, where the learning is over a \textit{single} trajectory without restarting the learning process by forcing the agent to return to $s_0$. Also, the regret definition is \textit{un-discounted} over the trajectory, and  $T$ could be any positive integer.


\section{Offline Planning}
This section studies the offline planning problem, where the goal is to find the optimal action policy given that the mean rewards $\{\mu_s:s\in S\}$ are known. 
State-of-the-art RL algorithms\cite{Jaksch2010a,agrawal2017optimistic} often use value iteration methods in their planning components. However, it is well-known that value iterations are computationally expensive since the number of iterations can be very large. In this section, we propose a much more efficient planning algorithm by formulating the offline planning problem as a shortest path (SP) problem. This planning algorithm will be used in our learning algorithm.

Define the expected loss/cost of visiting a node $s$ as 
$$c_s:=\mu^*-\mu_s=\mathbb{E}(\mu^*-r_s).$$ 


Let $\pi: S\rightarrow S$ be a policy that maps one node $s$ to a neighboring node $\pi(s)\in N_s$. We consider an infinite-time-horizon un-discounted offline planning problem as follows.
\vspace{-6pt}
\begin{equation}\label{eq:def_optim_problem_c_policy_inft}
	\begin{aligned}			
		\min_{\pi}&\  R_\infty(\pi,s_0):=\lim_{T\rightarrow \infty} \sum_{t=0}^T c_{s_t} \\
		~\text{such that}~& s_{t+1}=\pi(s_{t})\in N_{s_t},~\forall t=0,1,2,...
	\end{aligned}
\end{equation}
Note that any policy $\pi$ that transits from $s_0$ to $s^*$ in a finite number of steps has a finite $R_\infty(\pi,s_0)$ value, so Problem~\eqref{eq:def_optim_problem_c_policy_inft} has a finite optimal value. Furthermore, for any path that does not end up at $s^*$, $R_\infty(\pi,s_0)$ is infinite. Hence, the optimal path ends up at $s^*$. We could bound the optimal value of $R_\infty (\pi, s_0)$ by $D$ as shown in the following Lemma.
\begin{lemma}\label{lem:bound_Rinfty}
	The optimal value of Problem~\eqref{eq:def_optim_problem_c_policy_inft} is upper bounded by $D$ where $D$ is the diameter of the graph $G$ as in \eqref{eq:diameter}.
\end{lemma}
\begin{proof}
	The proof is straightforward by noticing that the shortest path policy $\pi$ with $s^*$ being the destination node incurs a bounded cost $R_\infty (\pi, s_0)\leq D$ regardless of $s_0$. This is because $r_s\in[0,1]$ for every node $s$ and $D$ is the graph diameter.
\end{proof}

By noting that the cost at $s^*$ is $0$, i.e., $c_{s^*}=0$, it is not difficult to see that Problem~\eqref{eq:def_optim_problem_c_policy_inft} could be viewed as a deterministic shortest-path problem~\cite{bertsekas2015dynamic} with the destination node being $s^*$. Algorithm \ref{alg:DP-offline} shows the pseudo-code for solving this shortest-path problem.
We define a weighted directed graph $\hat{G}=(S,\hat{\mathcal{E}},\mathcal{D})$ based on the original graph $G$, where $\hat{\mathcal{E}}$ are the directed edges converted from the un-directed edges in $E$. $\mathcal{D}$ stands for the distances for $(s,s')\in\hat{\mathcal{E}}$, where the distance from node $s$ to a neighboring node $s'\in N_s$ is $a_{ss'}:=c_{s'}$. We define the \textit{total distance} of an admissible path as the sum of distances between the consecutive nodes. 
The optimal policy $\pi_{\textbf{SP}}$ to Problem~\eqref{eq:def_optim_problem_c_policy_inft} is then the shortest path policy to $s^*$ on $\hat{G}$, which can be computed using any shortest-path algorithm such as Dijkstra's algorithm and Bellman-Ford algorithm \cite{cormen2022introduction}.  It is worth noting that $\pi_{\textbf{SP}}$ must take the agent to $s^*$ in no greater than $|S|$ steps since it cannot revisit any sub-optimal nodes.

\begin{algorithm}
	\caption{Offline SP planning algorithm 
	}
	\label{alg:DP-offline}
	\textbf{Input:} Graph $G=(S,\edges)$, mean reward vector $\mu=(\mu_1,\mu_2,\dots,\mu_{|S|})$
	
	\textbf{Output:} Policy $\pi_{\textbf{SP}}: S\rightarrow S$ 
	\begin{algorithmic}[1]
		\STATE $\mu^*\leftarrow\max(\mu_1,\ldots,\mu_{|s|})$,  \STATE $s^*\leftarrow \arg\max_s (\mu_1,\ldots, \mu_{|S|})$
		\STATE Define distance $a_{ss'}:=c_{s'}=\mu^*-\mu_{s'}$ for all $(s,s')\in \edges$
		\STATE Let $\hat{G}=(S,\hat{\edges},\mathcal{D})$ be a directed version of $G$ where $\mathcal{D}$ stands for the distances for $(s,s')\in \hat{\mathcal{E}}$.
		
		\STATE $\pi_{\textbf{SP}}\leftarrow \texttt{Dijkstra}(\hat{G},s^*)$, or $ \texttt{Bellman-Ford}(\hat{G},s^*)$ \\
		where $\pi_{\textbf{SP}}(s)$ is the node that leads to the path with the shortest total distance from $s$ to $s^*$.
		
	\end{algorithmic}
\end{algorithm}
\begin{remark}\label{rem:computation}
	\textup{The computation complexity of SP algorithms is $O(|S|\cdot|\edges|)$ in the worst case, but in practice, the total computation is often much less. 
		\ifshortpaperthenelse{The technical report\cite[Appendix F-B]{OurTechnicalReport}}{Appendix \ref{append:Computation}} shows that our algorithm under SP planning could run 1.6--11 times faster than the UCRL2 algorithm under value iteration.
	}
\end{remark}

\section{Optimistic Graph Bandit Learning}
This section presents a learning algorithm in graph bandit, where the agent knows the graph structure but not the reward distributions. The algorithm is built on the offline SP planning algorithm and the principle of optimism. 
We demonstrate that harnessing the graph bandit problem structure \textit{properly} makes learning more efficient than generic RL algorithms. 

\subsection{The Algorithm}

\begin{algorithm}[ht]
	\caption{\texttt{G-UCB}: Graph-UCB Algorithm }\label{alg:G-UCB}
	\begin{algorithmic}[1]
		\renewcommand{\algorithmicrequire}{\textbf{Input:}}
		\renewcommand{\algorithmicensure}{\textbf{Result:}}
		
		\REQUIRE
		The initial node $s_0$. The offline SP planning algorithm $\textbf{SP}$ in Algorithm~\ref{alg:DP-offline} that computes the optimal policy defined given the graph $G$ and a set of reward values $\{\hat{\mu}_s\}_{s\in S}$ for each node: $\hat{\pi} \gets \textbf{SP}(G,\{\hat{\mu}_s\}_{s\in S})$.
		
		\STATE $m\gets 0$
		\STATE  Follow any path that visits all nodes at least once.\texttt{//Initialize $U_0$.}
		
		\STATE Place the agent at $s_0$. $s_{curr}\gets s_0$.
		
		\WHILE{The agent hasn't received a stopping signal}
		
		\STATE $m\gets m+1$
		
		\STATE Calculate the UCB values $U_{m-1}(s)$ for all $s\in S$ according to \eqref{eq:UCB}.
		
		\STATE $\pi_m\gets \textbf{SP}(G,\{U_{m-1}(s)\}_{s\in S})$

		\WHILE{$U_{m-1}(s_{curr})< \max_s U_{m-1}(s)$}
		\STATE Execute $\pi_m$ for one step. $s_{curr}\gets \pi_m(s_{curr})$
		
		\STATE Collect reward at $s_{curr}$.
		\ENDWHILE
		
		\STATE $dest_m\gets s_{curr}$. Keep collecting rewards at node $dest_m$ until the number of its samples doubles compared to its sample number at the beginning of episode $m$.
		\texttt{//This operation ensures $n_m(dest_m) = 2n_{m-1}(dest_m)$.}

		\ENDWHILE
	\end{algorithmic}
\end{algorithm}

Algorithm \ref{alg:G-UCB} presents the pseudo-code for the online learning algorithm. Although the algorithm runs in a single trajectory without restarts, we divide the time steps into episodes with growing lengths for conceptual convenience. We define the UCB (upper-confidence-bound) value for each node according to the following rule, 
\vspace{-6pt}
\begin{equation}\label{eq:UCB}
	U_{m-1}(s) = \bar{\mu}_{m-1}(s)+ \sqrt{\frac{2\log(t_{m})}{n_{m-1}(s)}},
\end{equation}
where $t_{m}$ is the total number of time steps taken by the agent from episode $0$ to $m-1$ (episode $0$ is the initialization), $\bar{\mu}_{m-1}(s)$ is the average reward at $s$ and $n_{m-1}(s)$ is the number of samples of at node $s$ during the same period. 

The algorithm initializes the UCB values by visiting all nodes at least once. One possible method is to visit the nodes one by one, following paths with the shortest \textit{lengths}. 
At the beginning of each episode $m\geq 1$ (the end of episode $m-1$) after the initialization, the algorithm updates the UCB values with Eq.\eqref{eq:UCB}. Then the algorithm calls the offline SP planning algorithm $\bf{SP}$ in Algorithm~\ref{alg:DP-offline} to compute an optimal policy $\pi_m$, using the latest UCB values as the surrogates of the true expected rewards. The agent executes $\pi_m$ until it arrives at a node with the highest UCB value in the graph and then keeps sampling this node until the number of its samples doubles. 
Finally, the agent repeats the process by moving to the next episode. We denote the online learning algorithm as $\texttt{G-UCB}$.

\subsection{Main Theorems}
Theorem \ref{thm:G-UCB} states the theoretical guarantee on Algorithm \ref{alg:G-UCB}'s  learning regret. 
\begin{theorem}\label{thm:G-UCB}
	Let $T\geq 1$ be any positive integer. Recall that $|S|$ is the number of nodes, and $D$ is the diameter of the graph $G$ defined in \eqref{eq:diameter}.
	The regret of \texttt{G-UCB} (Algorithm \ref{alg:G-UCB}), after taking $T$ steps beyond the initialization is bounded by
	\begin{equation} \label{eq:thm-regret-1}
		\begin{aligned}
			R^*(\textup{\texttt{G-UCB}},s,T)
			\leq 
			O(\sqrt{|S|T\log(T)}+D|S|\log(T))
		\end{aligned}
	\end{equation}
	Note if furthermore $T\geq D^2|S|\log(T)$, the above reduces to
	$$
	\begin{aligned}
		&R^*(\textup{\texttt{G-UCB}},s,T)\leq O(\sqrt{|S|T\log(T)})
	\end{aligned}
	$$ 
\end{theorem}

\textit{Proof Sketch:}
We describe the outline of the proof below. The detailed proof is in \ifshortpaperthenelse{the technical report\cite[Appendix A]{OurTechnicalReport}}{Appendix \ref{appendix:thm-G-UCB}}. Define the random variable $M$ as the number of episodes at the cutoff time $T$. The doubling operation in line $12$ ensures $M$ is logarithmic in $T$ almost surely, i.e., $M\leq O(|S|\log(T))$,

For episode $m\geq 0$, we define the clean event $\mathcal{E}_m=$\{$\forall s\in S$, $\mu_s\in [\bar{\mu}_{m}(s)-\rad_{m}(s), \bar{\mu}_{m}(s)+\rad_{m}(s)]$\}, where $\rad_{m}(s) =  \sqrt{\frac{2\log(t_{m+1})}{n_m(s)}}$ are the confidence radii. Define  $^\neg\mathcal{E}_m$ as the bad event. We can decompose the regret via the clean/bad event distinction and use standard arguments involving Hoeffding's inequality(such as in ~\cite{slivkins2019introduction}) to show that the probabilities of the bad events are so small that the regret contributed by the bad events is bounded by $ O(\log(M)) = O(\log(|S|)+\log\log(T))$.

The regret under 
the clean events consists of two terms:
\vspace{-6pt}
\begin{equation}
	\begin{aligned}  \underbrace{2 \sum_{m=1}^M\sum_{t=1}^{H_m} \rad_{m-1}(s_{t_m+t})}_\textrm{price of optimism}+
		\underbrace{\sum_{m=1}^M\sum_{t=1}^{H_m} \mu^*-U_{m-1}(s_{t_m+t})}_{\textrm{cost of destination switch}} 
	\end{aligned}
\end{equation} 
The first component, `price of optimism', measures the regret from using the UCB as a surrogate for the actual mean rewards. 
\modified{To bound this component, we start by changing the summation over $t\in [1:H_m]$ to be a summation over $s\in S$.
	\begin{equation*}
		\begin{aligned}
			\sum_{m=1}^{M}\sum_{t=1}^{H_m}\rad_{m-1}(s_{t_m+t})=&
			\sum_{m=1}^{M}\sum_{t=1}^{H_m}\sqrt{\frac{2\log(t_{m})}{n_{m-1}(s_{t_m+t})}}\\
			\leq &\sqrt{2\log(T)}\sum_{m=1}^{M}\sum_{s\in S}\frac{c_m(s)}{\sqrt{n_{m-1}(s)}}
		\end{aligned}
	\end{equation*}
	Here, $c_m(s)$ denotes the number of samples at node $s$ during episode $m$. The summation 
	$\sum_{s\in S}\frac{c_m(s)}{\sqrt{n_{m-1}(s)}}$
	can be bounded using the inequality in \cite[Appendix C.3]{Jaksch2010a}, which states that for any sequence of numbers $z_1,z_2,...,z_n$ with $0\leq z_k\leq Z_{k-1}:=\max\{1,\sum_{i=1}^{k-1} z_i\}$, there is
	\begin{equation}\label{ineq:jacksh2010_main_text}
		\sum_{k=1}^n \frac{z_k}{\sqrt{Z_{k-1}}}\leq (\sqrt{2}+1)Z_{n}
	\end{equation}
	The proof of the eq. \eqref{ineq:jacksh2010_main_text} is in \ifshortpaperthenelse{the technical report\cite[Appendix C]{OurTechnicalReport}}{Appendix \ref{append:jacksh2010}}. In our case, 
	$c_{i-1}(s)$ corresponds to $z_i$ while $n_{i-1}(s)$ corresponds to $Z_i$. Since the \textbf{SP} policy does not visit a suboptimal node twice, while the doubling operation ensures the samples at $dest_m$ exactly double, there must be $c_m(s)\leq n_{m-1}(s)$. Therefore we can apply eq. \eqref{ineq:jacksh2010_main_text} and Jensen's inequality to get
	\begin{equation*}
		\begin{split}
			\sum_{m=1}^{M}&\sum_{s\in S}\frac{c_m(s)}{\sqrt{n_{m-1}(s)}} = \sum_{s\in S}\sum_{m=1}^{M}\frac{c_m(s)}{\sqrt{n_{m-1}(s)}}\\
			\leq& \sum_{s\in \verts} (\sqrt{2}+1)\sqrt{n_{M}(s)}\leq (\sqrt{2}+1) \sqrt{|\verts|T}\\
		\end{split}
	\end{equation*}
	
	So the first term is upper bounded by $O(\sqrt{|S|T\log(T)})$. 
}

The second component, `the cost of destination switch', measures the regret from visiting sub-optimal nodes during the transit to $dest_m$ in episode $m$. It can be shown that such transit induces at most $O(D)$ regret per episode.
So the total cost of the destination switch is bounded by $O(MD)=O(D|S|\log(T))$. The bound in the theorem is reached after combining all the components.\qed


\begin{remark}[Tightness of the result]\label{rem:MAB}
	\textup{In classical MAB, it is known that the regret is lower bounded by $\Omega(\sqrt{KT})$~\cite{slivkins2019introduction,Auer2002}, where $K$ is the number of arms.
		\modified{Note that on a fully connected graph, the graph bandit becomes equivalent to the MAB problem. Therefore, $\Omega(\sqrt{|S|T})$ is a lower bound to the graph bandit problem, and our regret $O(\sqrt{|S|T\log(T)})$ matches the lower bound up to logarithmic factors. This means our regret is the best possible.}
	}
\end{remark}

\begin{remark}[Comparison to state-of-the-art RL results]\label{rem:RL}
	\textup{Our result shows the benefit of utilizing the graph structure compared to state-of-the-art generic RL algorithms.
		A $O(D|S|\sqrt{|A|T\log(T)})$ regret bound is established in the general RL setting in \cite{Jaksch2010a}, where $|S|$, $|A|$ are the number of states and actions. This regret has a considerably worse dependence on $|S|$ than our regret bound $O(\sqrt{|S|T\log(T)})$, which is because their analysis has not utilized the graph bandit problem structures.  
		We found that, if assuming known transitions, their regret can be improved to $O((D+\sqrt{|S|A})\sqrt{T\log(T)})$ since the first term in their regret Eq. (19) vanishes. 
		This regret is still worse than our $O(\sqrt{|S|T\log(T)}$ regret due to the additional $\sqrt{|A|}$ factor, but we believe it could be further improved if their analysis carefully utilizes other graph bandit properties.
		Meanwhile, \cite{agrawal2017optimistic} establishes a $O(D\sqrt{|S||A|T\log(T)})$ regret. But in contrast to our work, they assume known rewards and unknown transition dynamics thus it is unclear what their regret will become under our settings.
	}
\end{remark}

\begin{remark}[Choice of Bonus function in UCB]\label{rem:UCB}
	\textup{
		The UCB in \cite{Jaksch2010a} is in defined as \begin{equation}\label{eq:UCRL2-UCB}
			\begin{aligned}
				\tilde{U}_{m-1}(s) &= \bar{\mu}_{m-1}(s)+\sqrt{\frac{7\log(SA t_m/\delta)}{2n_{m-1}(s)}},
			\end{aligned}
		\end{equation}  where $S, ~A$ are the number of states and actions, and $\delta\in(0,1]$ is a confidence parameter. The UCB definition above has an explicit dependence on $S$ and $A$, which is absent in our UCB(Eq. \eqref{eq:UCB}). This dependence on $S$ and $A$ arises in the confidence radius due to the analysis under the general RL setting, but our analysis under the graph bandit setting does not hint at such dependence. Directly applying this UCB to graph bandit should result in more exploration and higher regret than our UCB. This intuition is confirmed in the numerical experiments of \ifshortpaperthenelse{our technical report\cite[Appendix F-C]{OurTechnicalReport}}{Appendix \ref{append:improvement-by-UCB}}.}
\end{remark}
We also derive a regret bound that depends on the gap $\Delta$ between the best and second-best mean rewards; see Theorem \ref{thm:UCB-DEST-DELTA}. If all mean rewards are not the same, we define $\Delta:=\mu^* - \max_{s\in S,\mu_s<\mu^*}\mu_s$. \footnote{Regret bounds depending on $\Delta$ as in Theorem \ref{thm:UCB-DEST-DELTA} are known as instance-dependent bounds, while $\Delta$-independent bounds as in Theorem \ref{thm:G-UCB} are known as minimax bounds.}
\begin{theorem}[Instance-dependent bound]\label{thm:UCB-DEST-DELTA}
	Let $\Delta:=\mu^*-\hat{\mu}$ where $\hat{\mu}$ is the second best mean reward.  We also define $L$ as the maximal length of cycle-free paths in the graph. Note that $D\leq L\leq |S|$.
	Under the same assumptions as Theorem \ref{thm:G-UCB}, the regret of \textup{\texttt{G-UCB}} (Algorithm \ref{alg:G-UCB}) is bounded by.
	$$
	\begin{aligned}
		&R^*(\textup{\texttt{G-UCB}},s,T)\\
		\leq& O\bigg( |S|(L-1)[\log\log(T)+\log(\frac{1}{\Delta})] +\frac{|S|\log(T)}{\Delta}\bigg)
	\end{aligned}
	$$
\end{theorem}
See the \ifshortpaperthenelse{technical report\cite[Appendix B]{OurTechnicalReport}}{Appendix \ref{appendix:thm-UCB-DEST-DELTA}} for detailed proof. 

\begin{remark}\label{rem:thm-Delta}
	\textup{
		Instance-dependent bounds have been previously investigated under similar settings. In particular, the UCRL2 algorithm achieves regret $O(\frac{D^2|S|^3\log(T)}{\Delta})$\cite{Jaksch2010a} on connected graphs. If ignoring the $\log\log(T)$ and $\log(\frac{1}{\Delta})$ terms,  our results have much better  dependence on $|S|$ and $D$ than \cite{Jaksch2010a}.}
\end{remark}

\section{Numerical Experiments}\label{sec:numerical-experiments}
\subsection{Benchmark Experiments}\label{sec:benchmark-on-grid}
The simulation code can be found at \cite{Zhang_Graph-Bandit_2022}. Due to the space limit, we defer most numerical experiments to \ifshortpaperthenelse{our technical report\cite[Appendix F]{OurTechnicalReport}}{Appendix \ref{append:additional-simulations}}. Here, we only use the experiment on a $10\times 10$ grid graph as an example of demonstration, while the results for other graph types are similar and can be found in \ifshortpaperthenelse{the technical report\cite[Appendix F-A]{OurTechnicalReport}}{Appendix \ref{append:all-benchmark-all-graph}}. 
We compare our proposed algorithm to the following benchmark algorithms.

\begin{itemize}
	\item 
	Local UCB and Local TS. These benchmarks always move to the node with the highest UCB or posterior sampling value in the current neighborhood.
	
	\item
	QL-$\epsilon$ greedy\cite{sutton2018reinforcement} and QL-UCB-H\cite{ISQLPROVABLYEFFICIENT}. $Q$-learning algorithms with $\epsilon$-greedy exploration and state-of-the-art Hoeffding-style exploration bonus.
	
	\item
	UCRL2. For a fair comparison, we adjusted the original UCRL2 algorithm \cite{Jaksch2010a} so that this benchmark knows the deterministic transition dynamics.
	
\end{itemize}
We run 100 simulations for each algorithm on the $10\times 10$ grid graph, with $T=2\times 10^4$ steps per simulation. We initialize the mean rewards from $\mu_s \sim\mathcal{U}(0.5,9.5)$, for each node $s\in S$. In the simulations the reward distributions are defined as $P(s)=\mathcal{U}(\mu_s-0.5,\mu_s +0.5)$, $\forall s\in S$. Fig. \ref{fig:benchmark-on-grid} demonstrates a clear advantage of \texttt{G-UCB} over the benchmarks. 
UCRL2 performs much better than other benchmarks but is still not as good as \texttt{G-UCB}. \ifshortpaperthenelse{The technical report\cite[Appendix F-C]{OurTechnicalReport}}{Appendix \ref{append:improvement-by-UCB}} shows that the advantage of \texttt{G-UCB} over UCRL2 is mostly due to our more proper definition of UCB as in Remark \ref{rem:UCB}.

\begin{figure}[ht]
	\centering
	\includegraphics[width=0.6\linewidth]{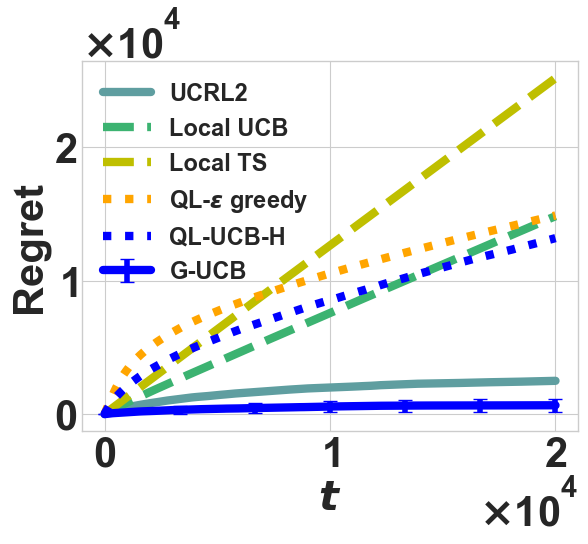}
	\caption{Learning regrets of the benchmarks on the grid graph. All regret curves in the figure indicate the average across $100$ simulations, and error bars indicate one standard deviation. We plot the error bars only for \texttt{G-UCB} for visualization clarity.}\label{fig:benchmark-on-grid}
\end{figure}

\subsection{Synthetic robotic application}\label{sec:robot-application}
This section presents a synthetic robotic application modeled as a graph bandit problem. Consider a drone traveling over a network of rural/suburban locations to provide internet access, as illustrated in Fig. \ref{fig:application}. The robot serves one location per period. The reward for serving the location is the communication traffic carried by the robot in gigabytes, sampled independently from a distribution associated with the location. The robot may stay at the current location or move to a neighboring location for the next period. It has a map of its service area but does not know the reward distributions. The goal is to maximize the total reward before the robot is called back and recharged. 
We perform numerical experiments on the graph defined in Fig. \ref{fig:application}, representing ten counties with lower than $50\%$ broadband coverage in AR, USA. We run our algorithm for $100$ simulations on this graph with the reward distributions initialized as in Fig. \ref{fig:benchmark-on-grid}. Fig. \ref{fig:applicationRegret} shows the robot consistently achieves sub-linear regrets in the simulations. \begin{figure}[ht]
	\centering
	\subfloat[Illustration]{\includegraphics[width=0.4\linewidth]{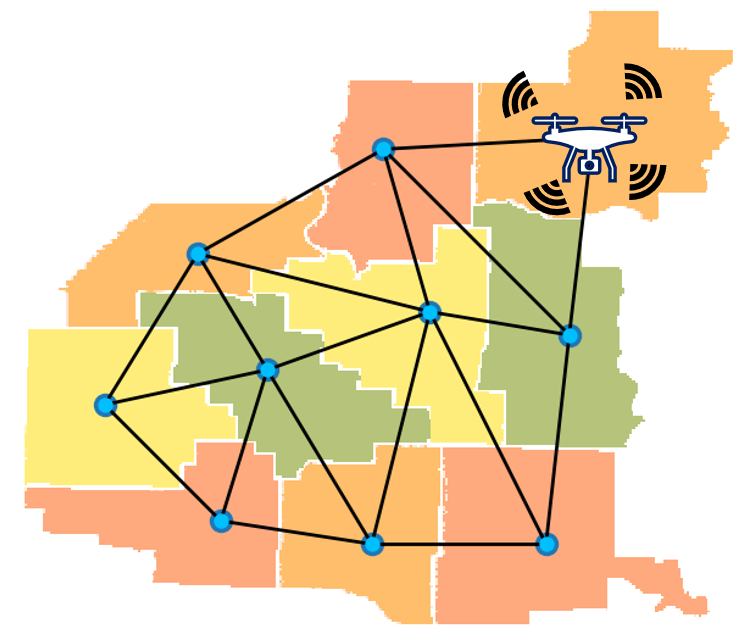}    \label{fig:application}}\quad
	\subfloat[Simulations]{\includegraphics[width=0.48\linewidth]{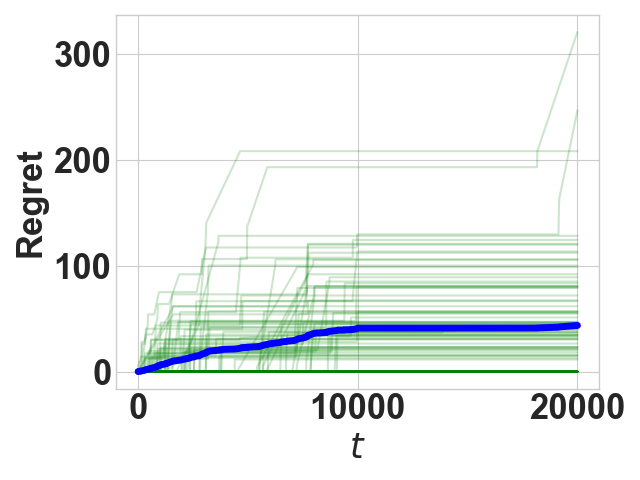}\label{fig:applicationRegret}}
	\caption{Robotic Application. The realized regrets of 100 simulations are in light green, and their average is in blue.
	}
\end{figure}

\section{Conclusion and Future Work}
This paper studies an extension of the MAB problem: the graph bandit problem, where the current selection determines the set of available arms through a graph structure. We show that offline planning can be converted into a shortest-path problem and solved efficiently. Based on the offline planning algorithm and the principle of optimism, we develop a UCB-based graph bandit learning algorithm, \texttt{G-UCB}, and establish a minimax regret bound for the algorithm that matches the theoretical lower bound.
Extensive numerical experiments verify our theoretical findings and show that our algorithm achieves better learning performance than the benchmarks. Our synthetic robotic application illustrates the real-world applicability of the graph bandit framework. Future directions include introducing time-varying reward distributions and/or extensions to the multi-agent settings.

\bibliographystyle{IEEEtran}
\bibliography{citations}

\clearpage
\ifshortpaperthenelse{\iffalse}{\iftrue}

\appendices
\section{Proof of Theorem \ref{thm:G-UCB}}\label{appendix:thm-G-UCB}
The proof of Theorem \ref{thm:G-UCB} is largely inspired by the techniques of the multi-armed bandit literature~\cite{slivkins2019introduction,lattimore2020bandit}. 
The relevant quantities in the proof are defined below. Let the initialization be episode $0$ of the algorithm. Denote $H_m$ as the number of steps taken in episode $m$ for any $m\geq 0$. Denote $\{t_m\}_{m\geq 1}$ as the times when the algorithm updates its mean reward estimation and confidence radii,
\begin{equation*}\label{eq:single-traj-tm}
	\begin{aligned}
		t_m &:= \sum_{i=0}^{m-1} H_i = t_{m-1}+H_{m-1},\forall m\geq 1\\
	\end{aligned}
\end{equation*}
, so that episode $m$ of the algorithm consists of the time steps $[t_{m}+1:t_{m+1}]$. Define $c_m(s)$ as the number of times $s$ is visited in episode $m$, and $n_m(s)$ counts the same number up to the end of episode $m$. 
\begin{equation*}\label{eq:single-traj-c-and-n}
	\begin{aligned}
		n_m(s) & := |\{i:s_i=s,0\leq i\leq t_{m+1}\}|,\forall m\geq 0\\
		c_m(s) & := |\{i:s_i=s,t_m+1\leq i\leq t_{m+1}\}|,\forall m\geq 0\\
	\end{aligned}
\end{equation*}

Define the confidence radius and estimated mean reward for node $s$ based on the data collected in episodes $[0:m]$ as
\begin{equation*}
	\begin{aligned}
		\rad_m(s)& := \sqrt{\frac{2\log(t_{m+1})}{n_m(s)}}\\
		\bar{\mu}_m(s)&= \frac{\sum_{i=1}^{n_m(s)}r_s^i}{n_m(s)}
	\end{aligned}
\end{equation*}
where $r_s^i$ is the reward received after visiting $s$ for the $i$'th time.

Define the UCB for node $s$ based on the observations up to the end of episode $m$ as
\begin{equation*}
	U_{m}(s) = \bar{\mu}_{m}(s)+\rad_m(s)
\end{equation*}

Define the random variable $M = \max\{m:t_m\leq T+t_1\} $ as the last episode before time step $T+t_1$, see Fig. \ref{fig:def-iterations} for an illustration. Let $dest_m$ be the terminal node of episode $m$, that is $dest_m=s_{t_{m+1}}$. 

\begin{figure}[ht]
	\centering
	\includegraphics[width=0.45\textwidth]{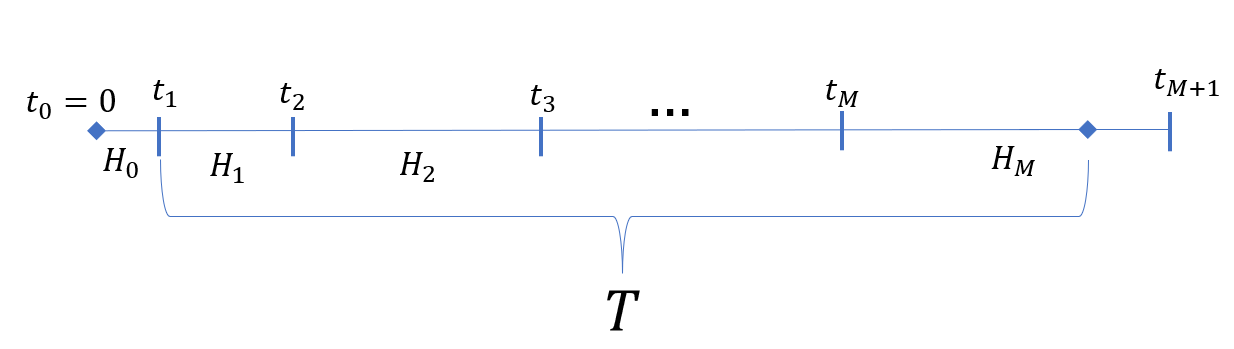}
	\caption{Relationship of the episodes and $M$ with respect to $T$.}
	\label{fig:def-iterations}
\end{figure}

\begin{proof}

	\hfil
	
	\textbf{Clean/bad event decomposition}
	
	\hfil
	
	The foundation of the proof is the distinction between the clean and bad events, which are defined as follows: let the clean event $\mathcal{E}_m$ be the event that the expected rewards fall into their corresponding confidence intervals for episode $m$, $$\begin{aligned}
		\mathcal{E}_m=&\{~\forall s\in S,\\
		&\mu_s\in [\bar{\mu}_m(s)-\rad_m(s),\bar{\mu}_m(s)+\rad_m(s)] \}
	\end{aligned}$$, so that the bad event $^\neg \mathcal{E}$ is the complement of the clean event, $$^\neg \mathcal{E}_m=\cup_{s\in S} \{|\bar{\mu}_m(s)-\mu_s|> \rad_m(s)\}.$$
	
	We use $\mathcal{F}_m = \sigma(H_{0:m})$ to denote the smallest $\sigma$-algebra where $H_{0:m}$ are measurable.  We will next consider the regret induced in a particular episode $m\geq 1$ condition on the episode lengths in $\mathcal{F}_m$.
	\begin{equation}\label{eq:cond-iteration-regret}
		\begin{aligned}
			&E[H_m\mu^*-V(s_{t_m+1:t_{m+1}})|\mathcal{F}_{m}]\\
		\end{aligned}
	\end{equation}
	Note that  $H_m$ and $t_m$ are measurable in $\mathcal{F}_{m}$.
	
	In what follows, we use $E_m[\cdot]$ as the equivalent of $E[\cdot|\mathcal{F}_{m}]$, and $P_m[\cdot]$ as the equivalent of $\pr[\cdot|\mathcal{F}_{m}]$ for notational simplicity. We proceed by decomposing eq. \eqref{eq:cond-iteration-regret} via the clean/bad event distinction.
	\begin{equation}
		\begin{aligned}
			&E_m[H_m\mu^*-V(s_{t_m+1:t_{m+1}})]\\
			=&E_m[H_m\mu^*-V(s_{t_m+1:t_{m+1}})|\mathcal{E}_{m-1}]P_m[\mathcal{E}_{m-1}]\\
			+&E_m[H_m\mu^*-V(s_{t_m+1:t_{m+1}})|^\neg \mathcal{E}_{m-1}]P_m[^\neg\mathcal{E}_{m-1}]\\
			\leq& E_m[H_m\mu^*-V(s_{t_m+1:t_{m+1}})|\mathcal{E}_{m-1}] + H_m P_m[^\neg\mathcal{E}_{m-1}]
		\end{aligned}
	\end{equation}
	
	The term $H_m P_m[^\neg\mathcal{E}_{m-1}]$ can be bounded using Hoeffding's inequality, as shown below.
	\begin{equation}
		\begin{aligned}
			&P_m[^\neg\mathcal{E}_{m-1}]\\
			\leq&\sum_{s\in S}P_m[|\bar{\mu}_{m-1}(s)-\mu_s|> \rad_m(s)],\text{ union bound.}\\
			\leq &\sum_{s\in S}\sum_{j=1}^{t_{m}}P_m[|\bar{\mu}_{m-1}(s)-\mu_s|> \sqrt{\frac{2\log(t_m)}{n_{m-1}(s)}}|n_{m-1}(s)=j]
		\end{aligned}
	\end{equation}
	In the second inequality, we use union bound over possible values of $n_{m-1}(s)$, and also the fact that $n_{m-1}(s)\in [1:t_{m}]$ for $m\geq 1$. Then note that
	\begin{equation}
		\begin{aligned}
			&P_m[|\bar{\mu}_{m-1}(s)-\mu_s|> \sqrt{\frac{2\log(t_{m})}{n_{m-1}(s)}}|n_{m-1}(s)=j]\\
			=& P_m[|\frac{\sum_{\tau=1}^j r_s^\tau}{j}-\mu_s|> \sqrt{\frac{2\log(t_{m})}{j}}]\\
			\leq & 2\exp(-2\frac{2j\log(t_{m})}{j}),\text{ Hoeffing's inequality. }\\
			=&2t_{m}^{-4}
		\end{aligned}
	\end{equation}
		
		
	
	Therefore,
	\begin{equation*}
		\begin{aligned}
			&H_mP_m[^\neg\mathcal{E}_{m-1}]       \leq  2H_m|S|t_{m}t_{m}^{-4}\leq 2H_m|S|t_{m}^{-3}
		\end{aligned}
	\end{equation*}
	
	Since the optimal path leading up to $dest_m$ is cycle-free, $dest_m$ must be reached after executing $\pi_m$ for $|S|$ steps, after which it is sampled for $n_{m-1}(dest_m)$. Therefore,
	\begin{equation}\label{eq:Hm_Bound}
		\begin{aligned}
			H_m\leq &|S|+n_{m-1}(dest_m)\leq |S|+t_m
		\end{aligned}
	\end{equation}
	Further using the fact that $t_m\geq t_1 \geq |S|$ due to initialization, we have
	
	\begin{equation}\label{eq:bad-event}
		\begin{aligned}
			H_m P_m[^\neg \mathcal{E}_{m-1}] \leq 4t_{m}^{-1}
		\end{aligned}
	\end{equation}
	Therefore, the unconditional regret induced in episode $m$ is bounded by:
	\begin{equation}
		\begin{aligned}
			&E[H_m\mu^*-V(s_{t_m+1:t_{m+1}})]\\
			=&E_{H_{0:m}}[E[H_m\mu^*-V(s_{t_m+1:t_{m+1}})|\mathcal{F}_m]]\\
			\leq &E[H_m\mu^*-V(s_{t_m+1:t_{m+1}})|\mathcal{E}_{m-1}]+4E[t^{-1}_{m}]
		\end{aligned}
	\end{equation}
	where $E_{H_{0:m}}$ means the expectation is taken over all possible values of $H_{0:m}$.
	
	Next, we re-write the component conditioned on $\mathcal{E}_{m-1}$ above as follows.
	\begin{equation}
		\begin{aligned}
			&E[V(s_{t_m+1:t_{M}})|\mathcal{E}_{m-1}]\\=&E[\sum_{t=1}^{H_m} \mu_{s_{t_m+t}}|\mathcal{E}_{m-1}]\\
			\geq& E[\sum_{t=1}^{H_m} \bar{\mu}_{m-1}(s_{t_m+t})-\rad_{m-1}(s_{t_m+t})|\mathcal{E}_{m-1}]\\
			,&\text{ definition of }\mathcal{E}_{m-1}.\\
			=&E[-2\sum_{t=1}^{H_m} \rad_{m-1}(s_{t_m+t})		+	\sum_{t=1}^{H_m} U_{m-1}(s_{t_m+t})|\mathcal{E}_{m-1}]\\
		\end{aligned}
	\end{equation}
	Therefore, conditioned on the the clean event $\mathcal{E}_{m-1}$, there is
	\begin{equation}\label{eq:85}
		\begin{aligned}
			&E[H_m\mu^* - V(s_{t_m+1:t_{m+1}})|\mathcal{E}_{m-1}]\\
			\leq &E[2\sum_{t=1}^{H_m} \rad_{m-1}(s_{t_m+t})+ \sum_{t=1}^{H_m}(\mu^*-U_{m-1}(s_{t_m+t}))|\mathcal{E}_{m-1}]
		\end{aligned}
	\end{equation}
	
	So the total regret is bounded by
	\begin{equation}
		\begin{aligned}
			&E[T\mu^*-V(s_{t_1:t_1+T})] \\
			\leq & E[(t_{M+1}-t_1)\cdot \mu - V(s_{t_1:t_{M+1}})]\\
			=&E\bigg( \sum_{m=1}^{M}H_m\mu^*-V(s_{t_m+1:t_{m+1}})\bigg)\\
		\end{aligned}
	\end{equation}
	
	Further combining with \eqref{eq:bad-event}, \eqref{eq:85} we have 
	\begin{equation}\label{eq:R*-decomposition}
		\begin{aligned}
			&E[T\mu^*-V(s_{t_1:t_1+T})] \\
			\leq & 2E[\sum_{m=1}^{M}E[\sum_{t=1}^{H_m}\rad_{m-1}(s_{t_m+t})|\mathcal{E}_{m-1}]] \\
			+&  E[\sum_{m=1}^{M}E[\sum_{t=1}^{H_m}(\mu^*-U_{m-1}(s_{t_m+t}))|\mathcal{E}_{m-1}]]\\
			&+4E[\sum_{m=1}^{M} t_{m}^{-1}]
		\end{aligned}
	\end{equation}
	
	We will bound each of the three terms in \eqref{eq:R*-decomposition}. Before that, we establish some important facts below.
	
	\begin{lemma}\label{lemma:tM+1}
		$t_{M+1}$ is at most $3(T+t_1)$.
	\end{lemma}
	\begin{proof}
		Consider episode $M$, where the agent travels for $H_M$ steps. By \eqref{eq:Hm_Bound}, we know $H_M\leq |S|+t_M$. Therefore
		\begin{equation*}
			\begin{aligned}
				t_{M+1} = & t_M+H_M\\
				\leq & 2t_M+|S|\\
				\leq & 3(T+t_1),\text{ since }t_M\leq T+t_1, |S|\leq t_1= H_0.
			\end{aligned}
		\end{equation*}
		Here $|S|\leq H_0$ because all nodes must be visited once during initialization.
	\end{proof}
	
	\begin{lemma}\label{lemma:M}
		$M$ is logarithmic in $T$. Specifically,
		\begin{equation*}
			M\leq |S|\frac{\log(3(T+t_1))}{\log(2)}
		\end{equation*}
	\end{lemma}
	\begin{proof}
		
		Consider any node $s\in S$. Let the episodes where $s$ is the terminal node be $$(M,m_2,...,m_{k(s)})= \{i\in [1:M]:dest_{m}= s\}$$, where $k(s)$ is the number of times $s$ becomes the destination node.  Line 12 of Algorithm \ref{alg:G-UCB} is designed so that if $s$ is $dest_m$, then it will be visited at least $n_{m-1}(s)$ times in episode $m$, doubling the number of its total visitation; otherwise, $s$ will be visited at most once in that episode. So for any $i\in [2:k(s)-1]$ there is
		\begin{equation*}
			n_{m_i}(s)= 2n_{m_{i}-1}(s) \geq 2 n_{m_{i-1}}(s)
		\end{equation*}
		And thus
		\begin{equation*}\label{eq:second-term-doubling}
			\begin{aligned}
				&\therefore n_{m_{i}}\geq 2^{i-1}n_{M}(s)\geq 2^{i}\\
				&\therefore n_{m_{k(s)}}\geq 2^{k(s)}
			\end{aligned}
		\end{equation*}
		where we used the fact that $n_{M}(s)\geq 2c_0(s)\geq 2$ due to initialization.
		
		Then note that the sum of $n_{m_{k(s)}}$ over $s$ is bounded by $t_{M+1}$, therefore
		\begin{equation}
			\begin{aligned}
				t_{M+1}\geq &\sum_{s\in S}n_{m_{k(s)}}\\
				\geq & \sum_{s\in S}2^{k(s)}\\
				\geq & |S|2^{\sum_{s\in S}\frac{k(s)}{|S|}},\textrm{ Jensen's inequality.}\\
				=&|S|2^{\frac{M}{|S|}},\textrm{ since }\sum_{s\in S}k(s) = M
			\end{aligned}
		\end{equation}
		Taking logarithms on both sides and simplifying,
		\begin{equation*}\label{eq:second-term-form-M}
			M\leq |S|\log(t_{M+1})/\log(2)
		\end{equation*}
		Plugging in $t_{M+1}\leq 3(T+t_1)$ completes the proof.
	\end{proof}
	
	\begin{lemma}\label{lemma:remove-clean-condition}
		For any random variable $x$ satisfying $x\geq 0$ almost surely, there is 
		\begin{equation*}
			E[x|\mathcal{E}_{m-1}]\leq 2E[x],\forall m\geq 1
		\end{equation*}
	\end{lemma}
	\begin{proof}
		Assume $m\geq 1$.
		
		\begin{equation*}
			\begin{aligned}
				E[x] = &E[x|\mathcal{E}_{m-1}]\pr[\mathcal{E}_{m-1}]+E[x|^\neg \mathcal{E}_{m-1}]\pr[^\neg \mathcal{E}_{m-1}]\\
				\geq &E[x|\mathcal{E}_{m-1}]\pr[\mathcal{E}_{m-1}], \textrm{ since }x\geq 0~a.s.
			\end{aligned}
		\end{equation*}
		Therefore,
		\begin{equation}
			\begin{aligned}
				E[x|\mathcal{E}_{m-1}]\leq &\frac{E[x]}{\pr[\mathcal{E}_{m-1}]}\\
				\leq & \frac{E[x]}{1-E_{H_{0:m}}[P_m(^\neg \mathcal{E}_{m-1})]}\\
				\leq &\frac{E[x]}{E[1-2|S|t_{m+1}^{-3}]}\\
				\leq &\frac{E[x]}{E[1-2t_{m+1}^{-2}]}\\
			\end{aligned}
		\end{equation}
		where in the last inequality, we used the fact that $|S|\leq t_1\leq t_{m+1}$ due to initialization. 
		
		Then note that there is also $$t_{m+1}\geq m+1\geq 2$$, since $dest_m$ must be visited at least once during $t_{m}+1:t_{m+1}$ for all $m$. Therefore, $1-2t_{m+1}^{-2} \geq \frac{1}{2}$, and 
		\begin{equation*}
			E[x|\mathcal{E}_{m-1}]\leq 2E[x],\forall m\geq 1
		\end{equation*}
	\end{proof}
	
	\textbf{The first term of \eqref{eq:R*-decomposition}:  the price of optimism.}
	
	Under the clean event, acting on the optimistic UCB value instead of the true expected reward typically induces a regret bounded by the confidence radius. Therefore, we can understand the second term, the sum of confidence radii, as `the price of optimism'. We will show this part of regret is bounded by $O(|S|\sqrt{T}\log(T))$.
	
	We start by getting rid of the $\mathcal{E}_{m-1}$ conditioning in the first term. Note that all the confidence radii are non-negative, so Lemma \ref{lemma:remove-clean-condition} applies, and we have
	\begin{equation}
		\begin{aligned}
			&2E[\sum_{m=1}^{M}E[\sum_{t=1}^{H_m}\rad_{m-1}(s_{t_m+t})|\mathcal{E}_{m-1}]]\\
			\leq & 4 E[\sum_{m=1}^{M}\sum_{t=1}^{H_m}\rad_{m-1}(s_{t_m+t})]\\
		\end{aligned}
	\end{equation}
	
	Consider the expression within the expectation.
	\begin{equation}\label{eq:single-trial-second-term}
		\begin{aligned}
			&\sum_{m=1}^{M}\sum_{t=1}^{H_m}\rad_{m-1}(s_{t_m+t})\\
			=&\sum_{m=1}^{M}\sum_{t=1}^{H_m}\sqrt{\frac{2\log(t_{m})}{n_{m-1}(s_{t_m+t})}}\\
			\leq &\sqrt{2\log(t_{M})}\sum_{m=1}^{M}\sum_{s\in S}\frac{c_m(s)}{\sqrt{n_{m-1}(s)}}
		\end{aligned}
	\end{equation}
	In the last inequality, we used the fact that $t_m$ is non-decreasing in $m$. We note that the term 
	$$\sum_{m=1}^{M}\sum_{s\in S}\frac{c_m(s)}{\sqrt{n_{m-1}(s)}}$$ 
	can be bounded using the inequality in Appendix C.3 of \cite{Jaksch2010a}, which states that for any sequence of numbers $z_1,z_2,...,z_n$ with $0\leq z_k\leq Z_{k-1}:=\max\{1,\sum_{i=1}^{k-1} z_i\}$, there is
	\begin{equation}\label{ineq:jacksh2010}
		\sum_{k=1}^n \frac{z_k}{\sqrt{Z_{k-1}}}\leq (\sqrt{2}+1)Z_{n}
	\end{equation}

	For completeness, we provide the proof of the Lemma above in Appendix \ref{append:jacksh2010}. In our case, for every $s\in \verts$, $c_0(s),c_1(s),c_2(s),...$ correspond to $z_1,z_2,z_3,...$ while $n_0(s),n_1(s),n_2(s),...$ correspond to $Z_1,Z_2,Z_3,...$. 
	
	Note that for any episode $m\geq 1$, any node $s$, $c_m(s)$ can only take values in 
	\begin{enumerate}
		\item $0$, if $s$ is not visited.
		\item $1$, is $s$ is visited but is not $dest_m$. This is because the optimal path leading to $dest_m$ does not contain cycles, thus $s$ cannot be visited twice if it is not $dest_m$.
		\item $n_{m-1}(s)$, if $s=dest_m$.
	\end{enumerate}
	Using the above along with the fact that $n_{m-1}(s)\geq 1$, there must be $c_m(s)\leq n_{m-1}(s)$ for all $m\geq 1$, $s\in S$. Therefore we can apply eq. \eqref{ineq:jacksh2010} with Jensen's inequality to get
	
	\begin{equation}
		\begin{split}
			\sum_{m=1}^{M}\sum_{s\in S}\frac{c_m(s)}{\sqrt{n_{m-1}(s)}} = &\sum_{s\in S}\sum_{m=1}^{M}\frac{c_m(s)}{\sqrt{n_{m-1}(s)}}\\
			\leq& \sum_{s\in \verts} (\sqrt{2}+1)\sqrt{n_{M}(s)}\\
			\leq& (\sqrt{2}+1) |\verts|\sqrt{\frac{1}{|\verts|}\sum_{s\in S}n_{M}(s)}\\
			=& (\sqrt{2}+1) \sqrt{|\verts|t_{M+1}}\\
		\end{split}
	\end{equation}
	
	So the first term is upper bounded by
	\begin{equation}
		\begin{aligned}
			&(\sqrt{2}+1)\sqrt{2|S|t_{M+1}\log(t_M)}\leq  O(\sqrt{|S|T\log(T)}) \\
		\end{aligned}
	\end{equation}
	
	Here we used the fact that $t_M\leq T,t_{M+1} \leq 3(T+t_1)$.

	
	

	
	\textbf{The second term of \eqref{eq:R*-decomposition}: the cost of destination switch.}
	
	We next show that each episode of the algorithm contributes a $D\mu^*$ regret to the third term, which captures the regret induced by visiting sub-optimal nodes during the transit from one terminal node to another one. The second term can be upper bounded by $DE[M]$, and thus $O(D|S|\log(T))$.

	To begin with, fix an episode $m\geq 1$. Let 
	\begin{equation}
		U^*_{m-1}:= \max_{s\in S} U_{m-1}(s)
	\end{equation} 
	be the highest $U_{m-1}$ value among all nodes,  thus $U_{m-1}(dest_m) = U^*_{m-1}$. Let $D_m = d(s_{t_m},dest_m)$ be the distance between the initial and terminal nodes of episode $m$. Note that there must be $H_m\geq D_m$. Then there exists a path $x_{0:H_m}$ that travels from $s_{t_m}$ to $dest_m$ in $D_m$ steps and stays at $dest_m$ afterwards. That is, 
	\begin{equation*}
		x_0 = s_{t_m}, x_{D_m:H_m} = dest_m
	\end{equation*}
	
	We will first show the cumulative $U_{m-1}$ value of $s_{t_m:t_{m+1}}$ must be no lower than that of $x_{0:H_m}$. 
	
	In what follows, we use
	the $\tilde{V}_m$ notation to denote the cumulative $U_m$ value for any path $s_{0:H}$.
	\begin{equation*}
		\tilde{V}_m(s_{0:H}) = \sum_{t=0}^H U_m(s_t)
	\end{equation*}
	We slightly abuse this notation to let $\tilde{V}_m(s,H)$ denote the optimal $H$-step cumulative UCB value starting at node $s$.
	\begin{equation*}
		\begin{aligned}
			\tilde{V}_m(s,H):=&\max_{s_{0:H}} \tilde{V}_m(s_{0:H})\\
			s.t.~&s_0=s\\
			~&s_{t+1}\in A_{s_t}, \forall t\in [0:H-1]
		\end{aligned}
	\end{equation*}
	
	If $U_{m-1}(s)=U^*_{m-1}$ for all $s$, then $\tilde{V}_{m-1}(s_{t_m:t_{m+1}})=\tilde{V}_{m-1}(x_{0:H_m})$ trivially. Therefore, we only consider the case where there is some $s$ such that $U_{m-1}(s)\neq U^*_{m-1}$, and let 
	\begin{equation*}
		\tilde{\Delta}_{m-1}: = U^*_{m-1}-\max_{s\in S, U_{m-1}(s)<U^*_{m-1}} U_{m-1}(s)
	\end{equation*}
	denote the gap between the highest and second highest $U_{m-1}$ values.
	
	Define $\tilde{H}_m:=\max\{H_m,\frac{DU^*_{m-1}}{\tilde{\Delta}_{m-1}}+|S|\}$, and let $\pi_m$ be the optimal policy with respect to $\{U_{m-1}\}_{s\in S}$, as defined in line $6$ of Algorithm \ref{alg:G-UCB}. Consider the path $p_{0:\tilde{H}_m}$ generated by starting at $s_{t_m}$ and applying $\pi_m$ for $\tilde{H}_m$ steps,
	\begin{equation*}
		p_0=s_{t_m};p_{t+1}=\pi_m(p_t), \forall t\in [0:\tilde{H}_m-1]
	\end{equation*}
	Essentially, $p_{0:\tilde{H}_m} = [s_{t_m:t_{m+1}},dest_m\times (\tilde{H}_m-H_m)]$.  By Proposition \ref{cor:long-term-optimal-converse} in Appendix \ref{appendix:policy-characterization}, the $R_\infty$ optimality of $\pi_m$ implies it is also optimal in the finite horizon $\tilde{H}_m$
	\begin{equation*}
		\tilde{V}_{m-1}(p_{0:\tilde{H}_m}) = \tilde{V}_{m-1}(s_{t_m},\tilde{H}_m)
	\end{equation*}
	Proposition \ref{cor:long-term-optimal-converse} can be derived using relatively basic techniques, and we defer its proof to Appendix \ref{appendix:policy-characterization}.
	
	
	

	Now we are ready to show $\tilde{V}_{m-1}(s_{t_m:t_{m+1}})\geq \tilde{V}_{m-1}(x_{0:H_m})$. The key is to use the fact that the path $[x_{0:H_m},(\tilde{H}_m-H_m)\times dest_m]$ is sub-optimal compared with $p_{0:\tilde{H}_m}$, and that $U_{m-1}(dest_m)=U^*_{m-1}$. 
	\begin{equation*}
		\begin{aligned}
			&\tilde{V}_{m-1}(x_{0:H_m})+(\tilde{H}_m-H_m)U^*_{m-1}\\
			=&\tilde{V}_{m-1}([x_{0:H_m},(\tilde{H}_m-H_m)\times dest_m])\\
			\leq& \tilde{V}_{m-1}(p_{0:\tilde{H}_m})\\
			=& \tilde{V}_{m-1}(s_{t_{m}:t_{m+1}})+(\tilde{H}_m-H_m)U^*_{m-1}\\
		\end{aligned}
	\end{equation*}
	
	Subtracting $(\tilde{H}_m-H_m)U^*_{m-1}$ from both sides of the above, we have
	\begin{equation*}
		\tilde{V}_{m-1}(s_{t_{m}:t_{m+1}})\geq \tilde{V}_{m-1}(x_{0:H_m})
	\end{equation*}
	That is to say,
	\begin{equation*}
		\begin{aligned}
			\sum_{t=1}^{H_m} U_{m-1}(s_{t_m+t}) \geq &\sum_{t=1}^{H_m} U_{m-1}(x_t) \\
			\geq &(H-D_m)U_{m-1}^*
		\end{aligned}
	\end{equation*}
	
	We then use the above to upper bound one episode in the third term as follows.
	\begin{equation*}
		\begin{aligned}
			&E[\sum_{t=1}^{H_m}(\mu^*-U_{m-1}(s_{t_m+t}))|\mathcal{E}_{m-1}]\\
			\leq& E[H_m\mu^*-(H_m-D_m)U_{m-1}^*|\mathcal{E}_{m-1}]\\
			=& E[(H_m-D_m)(\mu^*-U_{m-1}^*)+D_m\mu^*|\mathcal{E}_{m-1}]\\
			\leq& E[(H_m-D_m)(U_{m-1}(s^*)-U_{m-1}^*)+D_m\mu^*|\mathcal{E}_{m-1}]\\
			,&\textrm{ by the definition of clean event $\mathcal{E}_{m-1}$.}\\
			\leq& D\mu^*, \textrm{ since } D_m\leq D.
		\end{aligned}
	\end{equation*}
	where in the last inequality we also use the fact that $U_{m-1}(s^*)\leq U^*_{m-1}$. And therefore,
	\begin{equation*}
		\sum_{m=1}^{M}E[\sum_{t=1}^{H_m}(\mu^*-U_{m-1}(s_{t_m+t}))|\mathcal{E}_{m-1}]\leq MD
	\end{equation*}
	here we apply the fact that $\mu^*\leq 1$. Using the fact that $M\leq O(|S|\log(T))$, the second term is therefore upper bounded by 
	\begin{equation*}
		O(D|S|\log(T))
	\end{equation*}
	
	\textbf{The third term}
	
	The third term can be bounded easily using the fact that $t_{m+1}\geq m+1$.
	\begin{equation*}
		\begin{aligned}
			&4E[\sum_{m=1}^M t_{m}^{-1}]\\
			\leq &4 E[\sum_{m=1}^M m^{-1}]\\
			\leq & 4E[\log_2(M)+2]\\
			,&\textrm{ the logarithmic bound on harmonic sums.}\\
			\leq & O(\log(|S|)+\log\log(T))
		\end{aligned}
	\end{equation*}
	where in the last inequality we used $M\leq O(|S|\log(T))$.
	
	Notice the third term is completely dominated by the first two terms. So in summary, the total regret is upper bounded by
	\begin{equation*}
		O(\sqrt{|S|T\log(T)}+D|S|\log(T))
	\end{equation*}
	
\end{proof}

\section{Proof of Theorem \ref{thm:UCB-DEST-DELTA}}\label{appendix:thm-UCB-DEST-DELTA}
\begin{proof}
	Define the suboptimality gap for $s$ as
	\begin{equation*}
		\Delta(s) = \mu^*-\mu_s
	\end{equation*}
	
	Recall that $\Delta=\min_{s:\Delta(s)>0}\Delta(s)$. We also define $L$ as the maximal length of cycle-free paths in the graph. Note that $D\leq L\leq |S|$.
	
	Note that conditioned on $\mathcal{E}_{m-1}$, there is
	\begin{equation*}
		\begin{aligned}
			\Delta(dest_m)&\leq \rad_{m-1}(dest_m)  = \sqrt{\frac{2\log(t_{m})}{n_{m-1}(dest_m)}}\\
		\end{aligned}
	\end{equation*}
	
	Therefore, if $\Delta(s)>0$, there is
	\begin{equation*}
		s=dest_m\implies n_{m-1}(s)\leq \frac{2\log(t_{m})}{\Delta(s)^2}
	\end{equation*}
	In other words, if $n_{m-1}(s)> \frac{2\log(T)}{\Delta(s)^2}$ with $\Delta(s)>0$, then $s$ cannot be $dest_m$. Due to the doubling operation, we know that $n_{m-1}(s)$ grows exponentially fast, so intuitively any sub-optimal $s$ can be the destination node of an episode for at most $O(\log\left(\frac{2\log(T)}{\Delta(s)^2}\right))$ times.
	Base this insight, let $W(s)$ be the number of times a node $s$ becomes the destination node of an episode. 
	\begin{equation*}
		W(s) := |\{m\in[1:M]|dest_m = s\}|
	\end{equation*}
	Let $l(s):=\max \{m\in[1:M]|dest_m = s\}$ be the last episode where $s$ is the destination node. Conditioned on the event where $l(s)=l$ for some non-random number $l$ and $\mathcal{E}_{l(s)-1}$, the following holds due to the doubling operation.
	\begin{equation}
		\begin{aligned}
			\frac{2\log(T)}{\Delta(s)^2}\geq & n_{l(s)-1}(s) \\
			\geq &c_0(s)(1+2+2^2+...+2^{W(s)-1})\\
			\geq & 2^{W(s)} \\
			\therefore     W(s)\leq& \log\log(T) + 2\log(\frac{1}{\Delta(s)})
		\end{aligned}
	\end{equation}
	Then consider the regret induced in episode $m$. The optimality of $\pi_m$ in terms of $U_{m-1}$ implies 1) If $dest_{m-1} = dest_{m}=s^*$, then no regret is induced in episode $m$. 2) Otherwise, $dest_m = s_{t_{m+1}}$ is reached from $s_{t_m}$ in at most $L$ steps, since the optimal path contains no cycles. 
	Meanwhile, $dest_m$ is sampled $n_{m-1}(dest_m)$ times in episode $m$. 
	Therefore,
	\begin{equation}
		\begin{aligned}
			&\sum_{m=1}^M H_m\mu^*-V(s_{t_m+1:t_{m+1}})\\
			\leq &\sum_{m=1}^M \mathbf{1}\{dest_{m-1} \neq s^*\textup{ or }dest_{m}\neq s^*\}\\
			\cdot& [(L-1)+c_{m}(dest_m)\Delta(dest_m)]\\
			\leq &2\sum_{m=1}^M \mathbf{1}\{dest_m \neq s^*\}\\
			\cdot& [(L-1)+c_{m}(dest_m)\Delta(dest_m)]\\
			\leq &2\sum_{s\neq s^*}W(s) (L-1)+n_{l(s)}(s)\Delta(s)\\
		\end{aligned}
	\end{equation}
	Taking expectation on $W(s) (L-1)+n_{l(s)}(s)\Delta(s)$ for any $s$ in the above, conditioned on $\mathcal{F}_{l(s)}$ and the event where $l(s)=l$ for some non-random number $l$, there is
	\begin{equation}
		\begin{aligned}
			&E_{l(s)}[W(s) (L-1)+n_{l(s)}(s)\Delta(s)]\\
			\leq& 
			E_{l(s)}[W(s) (L-1)+n_{{l(s)}}(s)\Delta(s)|\mathcal{E}_{{l(s)}-1}]+t_{l(s)} P_{l(s)}[^\neg\mathcal{E}_{l(s)-1}]\\
			\leq& 
			E_{l(s)}[W(s) (L-1)+n_{{l(s)}}(s)\Delta(s)|\mathcal{E}_{{l(s)}-1}]+ 2|S|t^{-2}_{l(s)},\\
			& \textup{Hoeffding's inequality.}\\
			=& 
			E_{l(s)}[W(s) (L-1)+2n_{{l(s)}-1}(s)\Delta(s)|\mathcal{E}_{{l(s)}-1}]+ 2|S|t^{-2}_{l(s)},\\
			& \textup{Doubling operation.}\\
			\leq& 
			E_{l(s)}[W(s) (L-1)+\frac{4\log(t_{l(s)})}{\Delta(s)})|\mathcal{E}_{l(s)-1}]+ |S|t^{-2}_{l(s)}\\
			\leq& 
			(L-1)\left[\log\log(T)+2\log(\frac{1}{\Delta(s)})\right] +\frac{4\log(t_{M+1})}{\Delta(s)}\\
			&+1/|S|,\text{ since $t_{l(s)}\geq t_1\geq |S|$}\\
		\end{aligned}
	\end{equation}
	Note that the final inequality is independent of $l(s)$, $\mathcal{F}_{l(s)}$, so the conditioning on $\mathcal{F}_{l(s)}$ and the event where $l(s)=l$ for some non-random number $l$ can be dropped. Also, note that $t_{M+1}\leq O(T)$. So for any $s\in S/\{s^*\}$, there is
	
	\begin{equation*}
		\begin{aligned}
			&E[W(s) (L-1)+n_{l(s)}(s)\Delta(s)]\\
			\leq &(L-1)\left[\log\log(T)+2\log(\frac{1}{\Delta(s)})\right] +O(\frac{\log(T)}{\Delta(s)})
			+1/|S|
		\end{aligned}
	\end{equation*}
	
	The proof is complete after summing over $s$. 
	\begin{equation}
		\begin{aligned}
			&E[\sum_{m=1}^M H_m\mu^*-V(s_{t_m+1:t_{m+1}})]\\
			\leq & O(|S|(L-1)[\log\log(T)+\log(\frac{1}{\Delta})] +\frac{|S|\log(T)}{\Delta})
		\end{aligned}
	\end{equation}
	
\end{proof}

\section{Proof for the inequality in Eq. \eqref{ineq:jacksh2010}}\label{append:jacksh2010}
\begin{lemma}[\cite{Jaksch2010a}]
	For any sequence of numbers $z_1,z_2,...,z_n$ with $0\leq z_k\leq Z_{k-1}:=\max\{1,\sum_{i=1}^{k-1} z_i\}$, there is
	\begin{equation*}
		\sum_{k=1}^n \frac{z_k}{\sqrt{Z_{k-1}}}\leq (\sqrt{2}+1)Z_{n}
	\end{equation*}
\end{lemma}
\begin{proof}
	First, we show that the Lemma holds for those $n$ where $\sum_{i=1}^{n-1} z_i \leq 1$. Note if $\sum_{i=1}^{n-1} z_i \leq 1$, then $Z_k=1$ for all $k\leq n-1$, therefore
	\begin{equation*}
		\begin{split}
			\sum_{k=1}^n \frac{z_k}{\sqrt{Z_{k-1}}}\leq &\sum_{k=1}^{n-1} z_k +z_n \\
			\leq &2Z_{n-1} \\
			\leq &2 < (\sqrt{2}+1)\\
			\leq &(\sqrt{2}+1)\sqrt{Z_n}
		\end{split}
	\end{equation*}
	Even when $n=1$, the Lemma still holds since $z_1/\sqrt{Z_0} = z_1/1 \leq Z_1 \leq (\sqrt{2}+1)Z_1$.
	
	For those $n$ where $\sum_{i=1}^{n-1}z_i>1$, the Lemma can be proved by induction. In this case, $\sum_{i=1}^{n-1}z_i = Z_{n-1}$. Note that by induction hypothesis, there is
	\begin{equation*}
		\begin{split}
			\sum_{k=1}^n \frac{z_k}{\sqrt{Z_{k-1}}}\leq& (\sqrt{2}+1)\sqrt{Z_{n-1}} + \frac{z_n}{\sqrt{Z_{n-1}}}\\
			= &\sqrt{(\sqrt{2}+1)^2Z_{n-1}+2(\sqrt{2}+1)z_n+\frac{z_n^2}{Z_{n-1}}}\\
			\leq& \sqrt{(\sqrt{2}+1)^2Z_{n-1}+(2\sqrt{2}+3)z_n}\\
			&,\text{ since $z_n\leq Z_{n-1}$}\\
			= &\sqrt{(\sqrt{2}+1)^2(Z_{n-1}+z_n)}\\
			=& (\sqrt{2}+1)\sqrt{Z_n}
		\end{split}
	\end{equation*}
	So the proof is complete.
	
\end{proof}

\section{Additional Properties of Offline Planning}\label{appendix:policy-characterization}
The objective of this section is to show that the stationary policy $\pi_G$ characterized by the greedy action on $V(\cdot,T_G+|S|)$ produces the optimal actions for any horizon $T\geq T_G+|S|$; thus it is also optimal in $R_\infty$. Here $T_G:=\lceil \frac{D\mu^*}{\Delta}\rceil$. Moreover, a policy $\pi$ is optimal in $R_\infty$ if and only if it is optimal in the finite horizon $T_G+|S|$. We also discuss some useful properties that are relevant to offline planning.

\subsection{Sufficient Planning Horizon}
To begin with, we show that if the planning horizon $T$ is larger than \begin{equation}
	T_G:=\lceil \frac{D\mu^*}{\Delta}\rceil
\end{equation}, then the optimal path terminates in the nodes with the highest mean rewards. We denote $V(s_0,T)$ as the highest reward achievable until time $T$, starting at node $s_0$:
\begin{equation}\label{eq:def_optim_problem}
	\begin{aligned}
		V(s_0,T):=&\max_{s_1,s_2,...,s_T} V(s_{0:T})\\
		~such that~&s_{t+1}\in N_{s_t},~\forall t=0,1,...,T-1
	\end{aligned}
\end{equation}
\begin{lemma}[Sufficient Planning Horizon]\label{lemma:sufficient-T}
	\hfil
	
	If $T\geq \frac{D\mu^*}{\Delta}$ then $$\begin{aligned}
		&\forall s_0\in S, \exists \text{ a path } s_{0:T},~such that\\
		&\mu_{s_T}=\mu^*, V(s_{0:T})=V(s_0,T)
	\end{aligned}$$
	
	Moreover, if $T> \frac{D\mu^*}{\Delta}$, then $$\begin{aligned}
		&\forall s_0\in S, 
		V(s_{0:T})=V(s_0,T)\implies \mu_{s_T}=\mu^* 
	\end{aligned}$$
\end{lemma}
\begin{proof}
	Let $\textbf{s}=(s_0,s_1,...,s_T)$ be any path of length $T$. 
	
	Consider the following possibilities about this path: 1) it does not end up in one of the best nodes, i.e., $\mu_{s_T}<\mu^*$; 2) there is some step $t<T$ satisfying $\mu_{s_t}=\mu^*$. 
	
	If 1) and 2) are both true, then such a path cannot be the optimal one, since $V(s_{0:t},s_t\times T-t)$ is strictly greater than $V(s_{0:T})$. 
	
	Therefore, if $s_{0:T}$ achieves the optimal value, then either (a) $\mu_{s_T}=\mu^*$, or (b) $\mu_{s_{i}}<\mu^{*}$ for all $i\leq T$. 
	
	Hereafter, we assume $T\geq \frac{D\mu^*}{\Delta}$. If (b) holds then there must be \begin{equation}\label{eq:subopt-gap}
		\mu_{s_i}\leq \mu^* - \Delta
	\end{equation} for all $i=0,1,2,...,T$.
	
	Meanwhile, since we assume all $\mu_s$ are non-negative, there must be $\mu^*\geq \Delta$, and therefore $T\geq \frac{D\mu^*}{\Delta} \implies T\geq D$. Which means there is at least one path that starts at $s_0$ and reaches one of the best nodes $s^*$ in $D\leq T$ steps, where $\mu_{s^*}=\mu^*$. Consider one of such paths $$\textbf{w}=(s_0,w_1,...,w_{D-1},w_D=s^*,s^*\times T-D).$$ We will show $V(\textbf{w})\geq V(\textbf{s})$.
	
	\begin{equation}
		\begin{split}
			&V(s_0,w_{1:D},s^*\times(T-D))\\
			=&\mu_{s_0} +\sum_{i=1}^{D}\mu_{w_i} + (T-D)\mu^*\\
			\geq& (T-D)\mu^*+\mu_{s_0}, \text{ since $\mu_{w_i}\geq 0$}\\
			\geq&T\mu^*-T\Delta+\mu_{s_0},\text{ since $T\geq \frac{D\mu^*}{\Delta}$}\\
			\geq&\sum_{i=0}^T \mu_{s_i}=V(s_{0:T}),\text{ inequality \eqref{eq:subopt-gap}}
		\end{split}
	\end{equation}
	
	So the path $\textbf{w}=(s_0,w_{1:D},s^*\times{(T-D)})$ is no worse than $\textbf{s}=s_{0:T}$. And if $V(\textbf{s})=V(s,T)$, then $V(\textbf{w})=V(s,T)$ as well. So we have shown that if $T\geq \frac{D\mu^*}{\Delta}$, there is always a path starting at $s_0$ and ending at some $s^*$ with $\mu_{s^*}=\mu^*$ achieving the optimal value $V(s_0,T)$.
	
	Note if $T>\frac{D\mu^*}{\Delta}$ while (b) holds, then $(T-D)\mu^*+\mu_{s_0}>T\mu^*-T\Delta+\mu_{s_0}$, therefore $V(s_0,w_{1:D-1},s^*\times(T-D))>V(s_{0:T})$ and $s_{0:T}$ can not be optimal. That is to say, if $T>\frac{D\mu^*}{\Delta}$, then $V(s_{0:T})=V(s_0,T)\implies \mu_{s_T}=\mu^*$.
\end{proof}

We next show the following consequences of Lemma \ref{lemma:sufficient-T}(Sufficient Planning Horizon).

\begin{lemma}[$V$ increment]\label{lemma:V-increment}
	If $T\geq \frac{D\mu^*}{\Delta}$, then $V(s,T+1)=V(s,T)+\mu^*$ for any $s\in S$.
\end{lemma}
\begin{proof}
	Let $x_{0:T},w_{0:T+1}$ be paths satisfying $$x_0=w_0=s,~V(x_{0:T})=V(s,T),~V(w_{0:T+1})=V(s,T+1)$$. By Lemma \ref{lemma:sufficient-T}, $T\geq \frac{D\mu^*}{\Delta}$ implies $\mu_{x_T} = \mu_{w_{T+1}}=\mu^*$.
	
	Then notice the following:
	\begin{equation}
		\begin{aligned}
			V(s,T+1)=&V(w_{0:T+1})\\
			=&V(w_{0:T})+\mu^*\\
			\leq & V(s,T)+\mu^*
		\end{aligned}
	\end{equation}
	\begin{equation}
		\begin{aligned}
			V(s,T+1)\geq &V(x_{0:T},x_T)\\
			=&V(x_{0:T})+\mu^*\\
			=&V(s,T)+\mu^*\\
		\end{aligned}
	\end{equation}
	Therefore, $V(s,T+1)=V(s,T)+\mu^*$.
\end{proof}

\begin{corollary}[General V increment]\label{cor:V-increment}
	If $T\geq \frac{D\mu^*}{\Delta}$, then $V(s,T+N)=V(s,T)+N\mu^*$ for any $s\in S,N\geq 0$.
\end{corollary}

\begin{proof}
	Apply Lemma \ref{lemma:V-increment} repeatedly.
\end{proof}

\subsection{Terminal Nodes}\label{appendix:terminal-nodes}
The terminal nodes is a crucial concept in our analysis. Formally, we define the \textbf{terminal nodes} for an arbitrary path as follows:

\begin{definition}(Terminal Step \& Terminal Nodes)
	We define \textbf{terminal step} $K$ of an arbitrary path $s_{0:T}$ as
	\begin{equation}
		K = \min_{0\leq k\leq T}\{k:\mu_{s_t}=\mu_{s_k},~\forall t~such that~ k\leq t\leq T\}
	\end{equation}
	And define the \textbf{terminal nodes} of $s_{0:T}$ as $s_{K:T}$. Note $\mu_{s_{t}}=\mu_{s_K}$ for all $t$ such that $K\leq t\leq T$.
\end{definition}

The most fundamental terminal nodes results is stated in Lemma \ref{lemma:optimal-termination}: the optimal path must end up spinning in a node that has the highest mean reward among the entire path.
\begin{lemma}[Optimal Termination]\label{lemma:optimal-termination}
	Let $(s,s_1,s_2,...,s_T)$ be a path such that $V(s,s_{1:T})=V(s,T)$, then there is some $k\in [0:T]$, such that $\mu_{s_t} = \mu_{s_k}$ for all $T\geq t\geq k$, and $\mu_{s_k} \geq \mu_{s_t}$ for all $t\geq 0$.
\end{lemma}
\begin{proof}
	Let $\tilde{\mu}=\max_{t\geq 0} \mu_{s_t}$ be the highest mean reward in the path, and $k=\min\{t:\mu_{s_t}=\tilde{\mu}\}$ be the first time it is attained. 
	
	Note that the value of $(s,s_1,...,s_k,\underbrace{s_k,...,s_k}_{s_k\times (T-k)})$ is no worse than $(s,s_1,...,s_T)$, while the latter achieves the optimal value $V(s,T)$. 
	
	Therefore, $\mu_{s_t}=\tilde{\mu}=\mu_{s_k}$ for all $t\geq k$.
\end{proof}

\begin{corollary}[Strong Optimal Termination]\label{cor:strong-optimal-termination}
	If $\mathbf{s}=(s,s_{1:T})$ achieves value $V(s,T)$, and $K$ is the terminal step for $\mathbf{s}$ as defined above, then $\mu_{s_t}<\mu_{s_K}$ for all $t<K$.
\end{corollary}
\begin{proof}
	$K$ is a valid value for $k$ in Lemma \ref{lemma:optimal-termination}. While by the definition above, $K$ is the smallest of all $k$ such that $\mu_{s_t}=\mu_{s_k},~\forall t~such that~ k\leq t\leq T$. Therefore, if $t<K$, there must be $\mu_{s_t}<\mu_{s_K}$.
\end{proof}

One important fact about terminal nodes is that an optimal path must enter its terminal nodes in $|S|-1$ steps. That is, if a path is optimal, then its terminal step $K$ at at most $|S|-1$.
\begin{lemma}[Bounded Termination]\label{lemma:Bounded Termination}
	\textbf{Any} path $(s_0=s,s_1,...,s_T)$ with value $V(s,T)$ must enters its terminal nodes in no greater than $|S|-1$ steps. That is, $$\mu_{s_t}=\mu_{s_T}$$ for $t\in [|S|-1:T]$. 
	
	Moreover, the nodes leading up to the terminal step $K$ of $(s_0=s,s_{1:T})$ are all distinct, that is
	$$
	s_{l}\neq s_{h}~\forall l,h \textrm{ such that } 0\leq l<h\leq K
	$$
\end{lemma}
\begin{proof}
	
	Assume the path $\mathbf{p}=p_{0:T}$, with $p_0=s$, achieves $V(s,T)$. We will show that it is impossible for $p_{0:T}$ to enter the terminal nodes in more than $|S|-1$ steps.
	
	Define the following operation on $\mathbf{p}$:
	\begin{equation}
		\text{TRANSPLANT}(\mathbf{p},l,h) = (p_{0:l-1},p_{h:T},p_{K}\times(h-l))
	\end{equation}

	Suppose we can find $l<h\leq K$, such that $s_l = s_h$.	Consider the sequence of nodes $\mathbf{s}=\text{TRANSPLANT}(\mathbf{p},l,h)$.
	Divide $\mathbf{s,p}$ in the following manner.
	\begin{equation}
		\begin{split}
			\mathbf{s} &= ([p_{0:l-1}],[p_{h:T}],[p_K\times(h-l)])\\
			\mathbf{p}& =  ([p_{0:l-1}],[p_{l:h-1}],[p_{h:T}])
		\end{split}
	\end{equation}
	
	Then notice that
	\begin{equation}
		\begin{split}
			V(\mathbf{s})&= V(p_{0:l-1})+ V(p_{h:T})+ V(p_K\times(h-l))\\
			V(\mathbf{p})&= V(p_{0:l-1})+ V(p_{l:h-1})+ V(p_{h:T})
		\end{split}
	\end{equation}
	Note that, by Corollary \ref{cor:strong-optimal-termination}(Strong Optimal Termination), $l<h\leq K$ implies $V(p_{l:h-1})<V(p_K\times(h-l))$, which implies $V(\mathbf{s})>V(\mathbf{p})$ and contradicts with the optimality of $\mathbf{p}$.  Therefore, we have shown by contradiction that if $V(\mathbf{p})=V(s,T)$, then there are no $l<h\leq K$ satisfying $p_l=p_h$; in other words, the nodes in $p_{0:K}$ are all distinct.
	
	Finally, if terminal step $K$ is greater than $|S|-1$, then since there are $1+K>|S|$ nodes in $s_{0:K}$, by the Pigeon Hole principle, there must exist $l<h\leq K$ satisfying $p_l=p_h$, contradicting the distinctness property as shown above. Therefore, the terminal step $K$ of any optimal $\mathbf{p}$ must be no greater than $|S|-1$.
\end{proof}

\subsection{The equivalence of finite- and infinite-time optimal policies}

\begin{proposition}\label{thm:policy-characterization}
	Consider the following mapping $\pi_G$
	$$
	\pi_G(s) = \argmax_{w\in \neighbor(s)}V(w,T_G+|S|).
	$$
	Then for any $T\geq T_G+|S|$, $s_0\in S$, the path $s_{0:T}$ constructed by 
	\begin{equation}\label{eq:pi_G_trajectory}
		s_{t} \in \pi_G(s_{t-1}),t=1,2,...,T
	\end{equation}
	achieves the optimal value $V(s_0,T)$, where \eqref{eq:pi_G_trajectory} means $s_t$ can be any element in $\pi_G(s_{t-1})$.
\end{proposition}

\begin{proof}
	Fix a $s_0\in S$. Consider the potentially time-varying policy $\pi_t$
	\begin{equation}
		\pi_t(s) := \argmax_{w\in \neighbor(s)} V(w,T-t)
	\end{equation}
	
	Let $s_{0:T}$ be a path constructed by the standard value episode, which can be written as
	\begin{equation}
		s_t \in \pi_t(s_{t-1}) = \argmax_{w\in \neighbor(s_{t-1})} V(w,T-t),
	\end{equation} 
	thus $V(s_{0:T})=V(s_0,T)$. We will show that if $T\geq T_G+|S|$, then $\pi_t(s_{t-1})=\pi(s_{t-1})$ for all $1\leq t\leq T$, thus proving the theorem.
	
	Since $T\geq T_G+|S|> \frac{D\mu^*}{\Delta}$, by Lemma \ref{lemma:sufficient-T}(Sufficient Planning Horizon), there must be $\mu_{s_T}=\mu^*$. Furthermore, using Lemma \ref{lemma:Bounded Termination}(Bounded Termination), the $s_{0:T}$ enters the terminal nodes in $|S|-1$ steps. 
	Therefore,
	\begin{equation}
		\mu_{s_{t}}=\mu_{s_T} = \mu^*,
		\forall  t\in [|S|-1:T]
	\end{equation}
	
	which means for any $k\geq 0$, $ t\in [|S|:T]$, there is
	\begin{equation}
		\begin{split}
			&\argmax_{w\in \neighbor(s_{t-1})} V(w,k) = \{w\in \neighbor(s_{t-1}):\mu_w = \mu^*\}\\
		\end{split}
	\end{equation}
	In particular, when $k=T-t$ and $k=T_G+|S|$, 
	\begin{equation}
		\begin{aligned}
			&\argmax_{w\in \neighbor(s_{t-1})} V(w,T-t)
			=\argmax_{w\in \neighbor(s_{t-1})} V(w,T_G+|S|) 
		\end{aligned}
	\end{equation}
	Therefore for all $t\geq |S|$, $\pi_t(s_{t-1})=\pi(s_{t-1})$.
	
	Next, we consider the case where $1\leq t\leq|S|-1$. Using using Corollary \ref{cor:V-increment},  the following holds for any subset of nodes $B\in S$,
	\begin{equation}
		\begin{split}
			&\argmax_{w\in B} V(w,T-t)\\
			=&\argmax_{w\in B} V(w,T_G)+(T-t-T_G)\mu^*,\text{ Corollary \ref{cor:V-increment}}\\
			=&\argmax_{w\in B} V(w,T_G)\\
			=&\argmax_{w\in B} V(w,T_G)+|S|\mu^*,\text{ Corollary \ref{cor:V-increment}}\\
			=&\argmax_{w\in B} V(w,T_G+|S|)
		\end{split}
	\end{equation}
	
	Taking $B=\neighbor(s_{t-1})$, we have $\pi_t(s_{t-1})=\pi(s_{t-1})$ for $1\leq t\leq |S|-1$.
	
	Therefore, we have shown $\pi_t(s_{t-1})=\pi(s_{t-1})$ for all $t$, that is to say $s_{0:T}$ can be equivalently generated by the stationary policy $s_t \in \pi(s_{t-1})$. Since all the argument above holds for arbitrary $s_0$, the proof is complete.
\end{proof}

By letting $T\rightarrow\infty$, we can easily deduce that the policy $\pi_G$ stated in the Proposition is a policy achieving the smallest $R_\infty$ for any initial node $s$. 

\begin{proposition}
	\label{cor:long-term-optimal-equivalence}
	$$R_\infty(\gb,\pi_G,s) = \min_{\pi'} R_\infty(\gb,\pi',s),~\forall s\in S$$
\end{proposition}
\begin{proof}
	We use Proposition \ref{thm:policy-characterization} to show $\pi_G$ achieves the minimum of $R_\infty$. Consider any $T\geq T_G+|S|$, any $s\in S$, and any policy $\pi'$. Define 
	\begin{equation}
		R(\gb,\pi',s,T):= T\mu^* - E_{s_0=s,s_{t+1}\sim \pi'(s_t)}[V(s_{1:T})]
	\end{equation}
	Then 
	\begin{equation}
		\begin{aligned}
			&R(\gb,\pi',s,T)\\
			=&T\mu^* - E_{s_0=s,s_{t+1}\sim \pi'(s_t)}[V(s_{1:T})]\\
			\geq &T\mu^*-(V(s,T)-\mu_s),\textrm{ since }V(s_{1:T})\leq V(s,T)-\mu_s\\
			=& T\mu^*-(E_{s_0=s,s_{t+1}\in \pi_G(s_t)}[V(s_{0:T})]-\mu_s),\textrm{ Proposition \ref{thm:policy-characterization}}\\
			=& T\mu^*-E_{s_0=s,s_{t+1}\in \pi_G(s_t)}[V(s_{1:T})]\\
			=& R(\gb,\pi_G,s,T)
		\end{aligned}
	\end{equation}
	
	In other words, $R(\gb,\pi',s,T)\geq R(\gb,\pi_G,s,T)$ for any policy $\pi'$, any $s\in S$, and $T\geq T_G+|S|$. Letting $T\rightarrow \infty$ and taking the minimum on both sides gives us 
	\begin{equation}
		\min_{\pi'} R_\infty(\gb,\pi',s)=R_\infty(\gb,\pi_G,s),~\forall s\in S
	\end{equation}

\end{proof}

Finally, if a policy $\pi$ is optimal for $R_\infty$, then it has to be optimal for any finite horizon $T\geq T_G+|S|$, thus we establish an equivalence between $\pi_G$ and the $R_\infty$-optimal policy.
\begin{proposition}
	\label{cor:long-term-optimal-converse}
	Consider any $s\in S$. If $\pi$ satisfies $R_\infty(\gb,\pi,s) = \min_{\pi'} R_\infty(\gb,\pi',s)$, then $\pi$ is optimal for any finite time horizon $T\geq T_G+|S|$
\end{proposition}
\begin{proof}
	
	For the other direction, fix any $s\in S$, and take any $\pi\in \arg\min_{\pi'} R_\infty(\gb,\pi',s)$. Fix any time horizon $T\geq T_G+|S|$. Consider the path $x_{0:T}$ with $x_0=s$, generated by executing $\pi$ for $T$ steps. Similarly, let $s_{0:T}$ with $s_0=s$ be the path generated by $\pi_G$. Since $\pi$ is not necessarily optimal for finite horizon $T$, there is
	\begin{equation}\label{eq:73}
		E_\pi[V(x_{1:T})]\leq E_{\pi_G}[V(s_{1:T})]=V(s,T)-\mu_s,\text{Proposition \ref{thm:policy-characterization}}
	\end{equation}
	
	So 
	\begin{equation}\label{eq:74}
		\begin{aligned}
			&R_\infty(\gb,\pi,s)\\
			\geq&T\mu^*-E_\pi[V(x_{1:T})]\\
			\geq &T\mu^*-E_{\pi_G}[V(s_{1:T})]\\
			= & R_\infty(\gb,\pi_G,s)
		\end{aligned}
	\end{equation}
	The last equality holds since $R(\gb,\pi_G,s,T_G+|S|)=R_\infty(\gb,\pi_G,s)$ as shown in Proposition \ref{cor:long-term-optimal-equivalence}. Note that $$R_\infty(\gb,\pi,s)=R_\infty(\gb,\pi_G,s)$$ since $\pi,\pi_G$ are both optimal in $R_\infty$. So the inequalities in \eqref{eq:74} are equalities, thus $E_\pi[V(x_{1:T})]=E_{\pi_G}[V(s_{1:T})]=V(s,T)-\mu_s$, meaning $\pi$ achieves the optimal value for finite time horizon $T$. Finally, the proof is complete since $s$ and $T$ are arbitrarily selected.
\end{proof}
\section{The implementation of UCRL2 under graph bandit setting}\label{append:UCRL2-GB-Detail}

This section discusses how we modify the original UCRL2 algorithm in the graph bandit setting. The pseudocode of our implementation is in Alg \ref{alg:UCRL2-GB} and \ref{alg:VI-offline}. 

The original algorithm UCRL2\cite{Jaksch2010a} consists of two key steps. The first step is offline planning based on Extended Value Iteration(EVI), where optimism is applied to compute a policy for the agent. The second step is a doubling scheme slightly different from \texttt{G-UCB}. When applying UCRL2 to graph bandit as in Alg \ref{alg:UCRL2-GB}, the original doubling scheme is kept(line 11-14); only the EVI is changed to standard value iteration(\textbf{VI}) defined in Alg. \ref{alg:VI-offline}(line 10), with the merit for each node being the UCB value defined in line 9.
\begin{algorithm}[ht]
	\caption{Implementation of UCRL2 under graph bandit setting}\label{alg:UCRL2-GB}
	\begin{algorithmic}[1]
		\renewcommand{\algorithmicrequire}{\textbf{Input:}}
		\renewcommand{\algorithmicensure}{\textbf{Result:}}
		
		\REQUIRE
		The initial node $s_0$. Confidence parameter $\delta$ for the UCB. The value iteration planning algorithm $\textbf{VI}$ defined in Alg. \ref{alg:VI-offline}, which computes an $\epsilon$-optimal policy given the graph $G$ and a set of reward values $\{\hat{\mu}_s\}_{s\in S}$ for each node: $\hat{\pi} \gets \textbf{VI}(G,\{\hat{\mu}_s\}_{s\in S},\epsilon)$. 
		
		\STATE $m\gets 0$
		\STATE  Follow any path that visits all nodes at least once.\texttt{//Initialize $\bar{\mu}_0,\rad_0$.}
		
		\STATE Place the agent at $s_0$. $s_{curr}\gets s_0$.
		
		\WHILE{The agent hasn't received a stopping signal}
		
		\STATE $t_m\gets$ the total number of steps before this episode. 
		\STATE $n_{m-1}(s)\gets $the number of times $s$ is visited before this episode, for all $s\in S$.
		\STATE $c_m(s)\gets 0$ for all $s\in S$.
		\STATE  $m\gets m+1$. 
		
		\STATE Calculate the UCB values $$\tilde{U}_{m-1}(s) = \bar{\mu}_{m-1}(s)+\sqrt{\frac{7\log(SA t_m/\delta)}{2n_{m-1}(s)}}$$ for all $s\in S$.
		
		\STATE $\pi_m\gets \textbf{VI}(G,\{\tilde{U}_{m-1}(s)\}_{s\in S},\frac{1}{\sqrt{t_m}})$ 
		
		\texttt{//The optimality threshold at iteration $m$ is $\frac{1}{\sqrt{t_m}}$.}

		\WHILE{$c_m(s_{curr})< n_{m-1}(s_{curr})$}

		\STATE Collect reward at $s_{curr}$.
		\STATE $c_m(s_{curr})\gets c_m(s_{curr})+1$.
		\ENDWHILE
		
		\texttt{//The above inner loop ensures the number of samples at some node is doubled in this episode.}
		\STATE Execute $\pi_m$ for one step. $s_{curr}\gets \pi_m(s_{curr})$
		\ENDWHILE
	\end{algorithmic}
\end{algorithm}

\begin{algorithm}[ht]
	\caption{\textbf{VI}: Value Iteration planning algorithm 
	}
	\label{alg:VI-offline}
	\textbf{Input:} Graph $G=(S,\edges)$, mean reward vector $\mu=(\mu_1,\mu_2,\dots,\mu_{|S|})$, optimality threshold $\epsilon$.
	
	\textbf{Output:} Policy $\pi: S\rightarrow S$ 
	\begin{algorithmic}[1]
		\STATE $u_0(s)\gets 0$ for all $s\in S$.
		\STATE $i\gets 0$
		\REPEAT
		\STATE $i\gets i+1$
		\STATE For all $s\in S$, $$u_i(s)\gets \mu_s + \max_{s'\in N_s} u_{i-1}(s').$$
		\UNTIL{$$\max_{s\in S}(u_{i}(s)-u_{i-1}(s)) - \min_{s\in S}(u_{i}(s)-u_{i-1}(s))<\epsilon$$}    
		\STATE Return $\pi$ defined as $$\pi(s) \in \argmax_{s'\in N_s} u_i(s'),~\forall s\in S $$
	\end{algorithmic}
\end{algorithm}
\section{Full Numerical Experiment Results}\label{append:additional-simulations}

The code for the numerical experiments can be found at \cite{Zhang_Graph-Bandit_2022}.

\subsection{All benchmarks on every type of graph}\label{append:all-benchmark-all-graph}
\begin{figure}[ht]
	\centering
	\subfloat[FC.]{\includegraphics[width=0.16\linewidth]{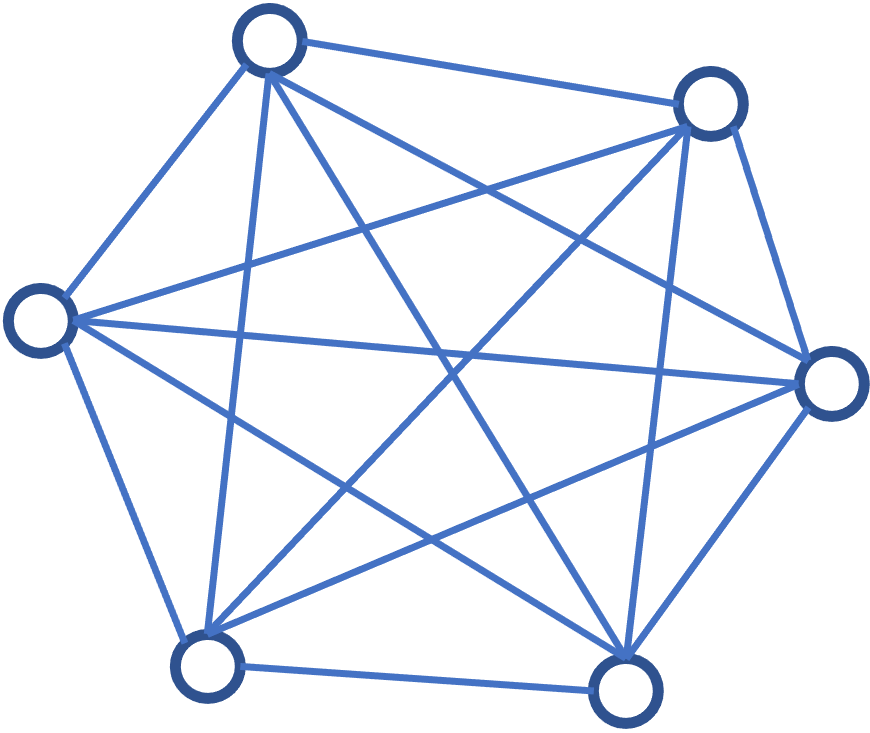}}
	\subfloat[Line.]{\includegraphics[width=0.16\linewidth]{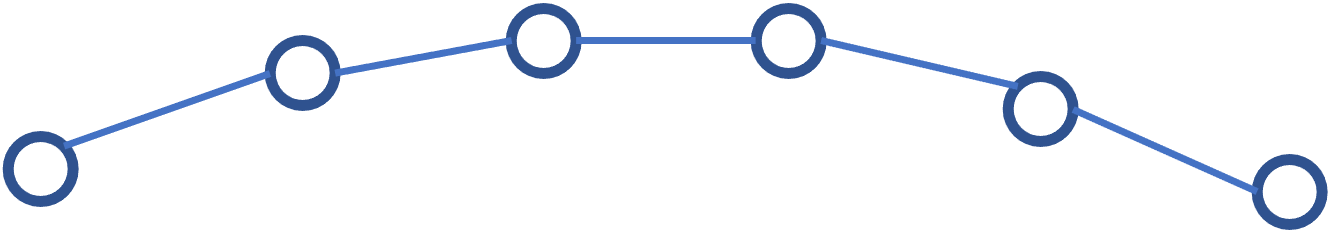}}
	\subfloat[Circle.]{\includegraphics[width=0.16\linewidth]{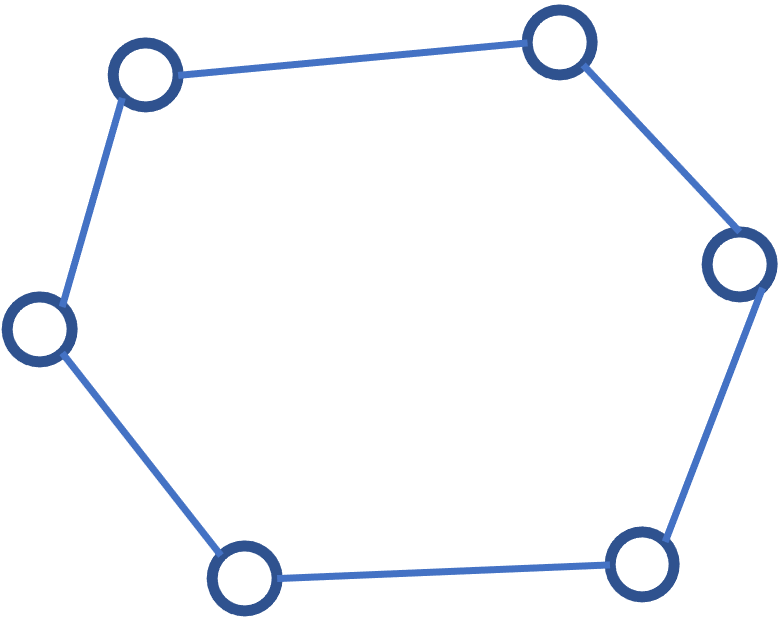}} 
	\subfloat[Star.]{\includegraphics[width=0.16\linewidth]{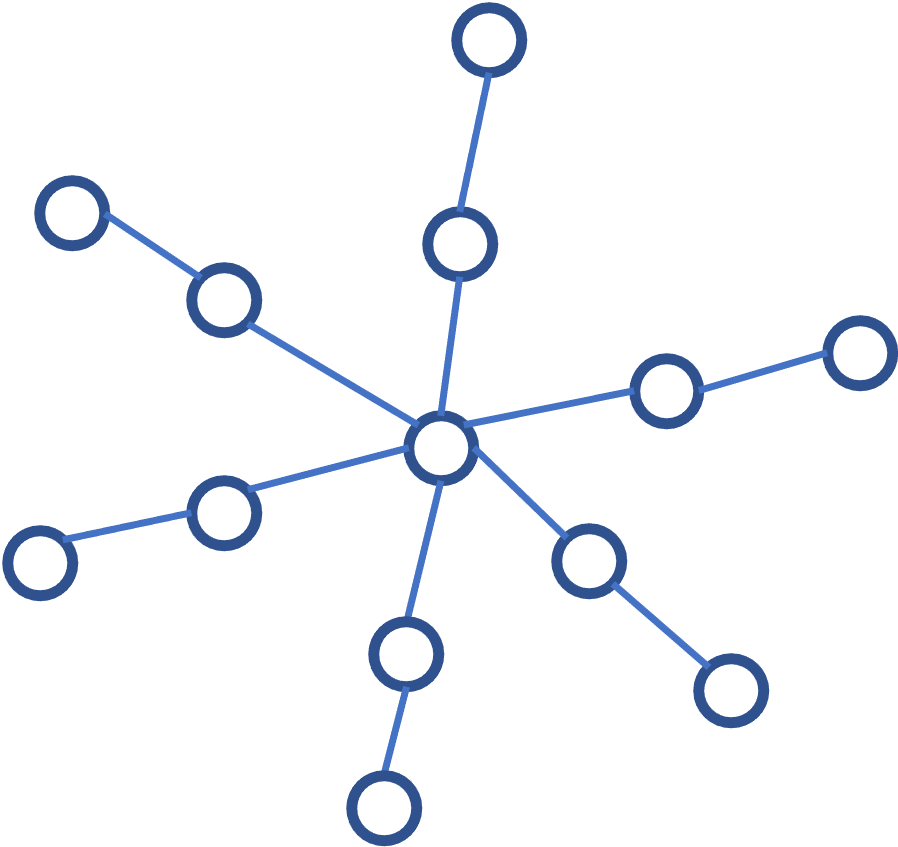}}\hspace{0.5mm}
	\subfloat[Tree.]{\includegraphics[width=0.16\linewidth]{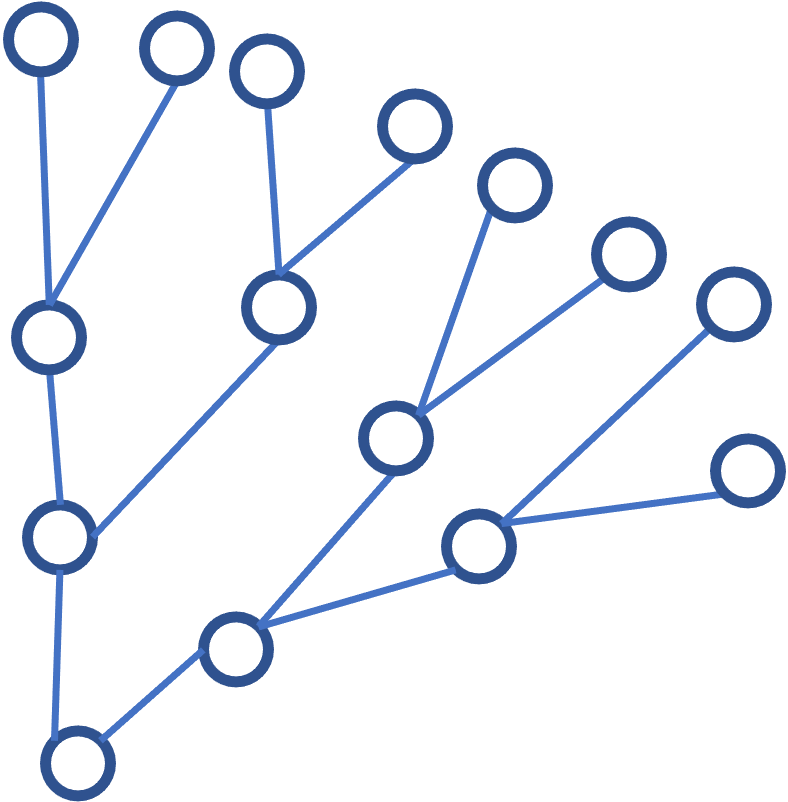}}\hspace{0.5mm}
	\subfloat[Grid.]{\includegraphics[width=0.16\linewidth]{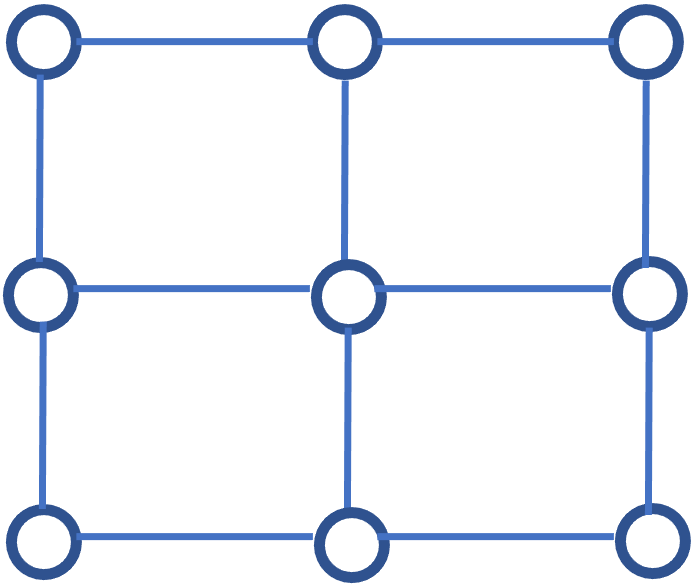}}
	\caption{The graph types in the experiments. FC means `fully connected' in (a).}
	\label{fig:graphs}
\end{figure}
Fig. \ref{fig:graphs} illustrates the six graph types we considered in the numerical experiments. We run 100 simulations for each algorithm on every graph, with $T=2\times 10^4$ steps per simulation. For the fully connected graph, we initialize the mean rewards from $\mu_s \sim\mathcal{U}(0.5,1.5)$, for each arm $s\in S$; for the remaining graphs, we initialize the mean rewards from $\mu_s \sim\mathcal{U}(0.5,9.5)$, for each node $s\in S$. In the simulations the reward distributions are defined as $P(s)=\mathcal{U}(\mu_s-0.5,\mu_s +0.5)$, $\forall s\in S$. 

Fig. \ref{fig:ALL} shows the regrets for all benchmark algorithms on every type of graph. Overall, \texttt{G-UCB} shows a clear advantage over the benchmarks.  We observe that local UCB and TS often get trapped in the `local maxima of rewards' due to a lack of foresight, making their regrets grow almost linearly. The two Q-learning benchmarks achieve sub-linear regrets, but these model-free algorithms need more samples than \texttt{G-UCB} to learn the optimal policy because they do not use the graph structure. \texttt{G-UCB} thus outperforms these benchmarks by mitigating the `local maxima' with foresight while staying sample efficient using problem knowledge. 

Fig. \ref{fig:GUCB-vs-UCRL2} shows a head-to-head comparison  between \texttt{G-UCB} and UCRL2. The regret of UCRL2 within one standard deviation to \texttt{G-UCB} on the line and circle graph, which is rather close, but the advantage of \texttt{G-UCB} is prominent on the remaining graphs, where the two algorithms are separated by two SD. As shown next in Section \ref{append:improvement-by-UCB}, the improvement of \texttt{G-UCB} from UCRL2 is predominantly brought by our more proper definition of UCB using the understanding of the graph bandit problem structure.
\begin{figure}[ht]
	\centering
	\subfloat[Line]{\includegraphics[width=\benchmarkwidth]{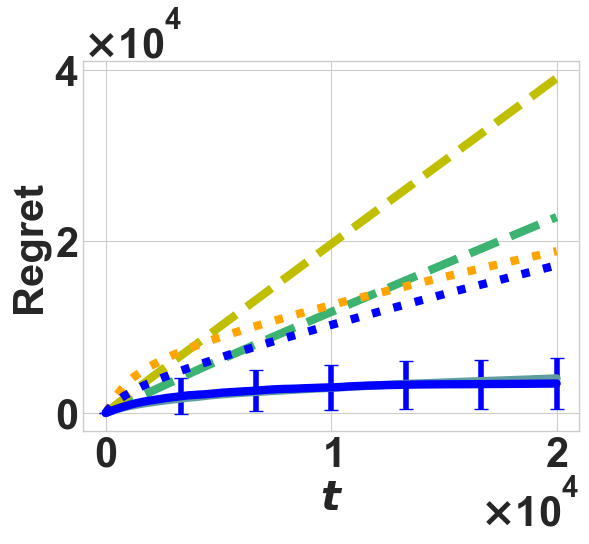} }
	\subfloat[Circle]{\includegraphics[width=\benchmarkwidth]{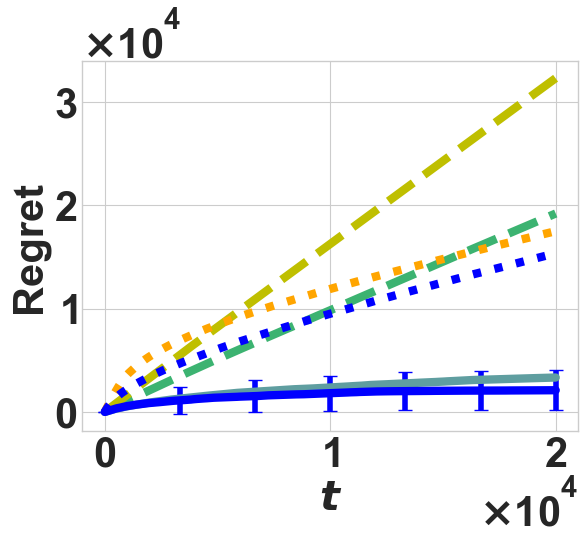} }  
	\subfloat[Grid]{\includegraphics[width=\benchmarkwidth]{Figures/benchmarks/grid.png}}
	\\
	\subfloat[Tree]{\includegraphics[width=\benchmarkwidth]{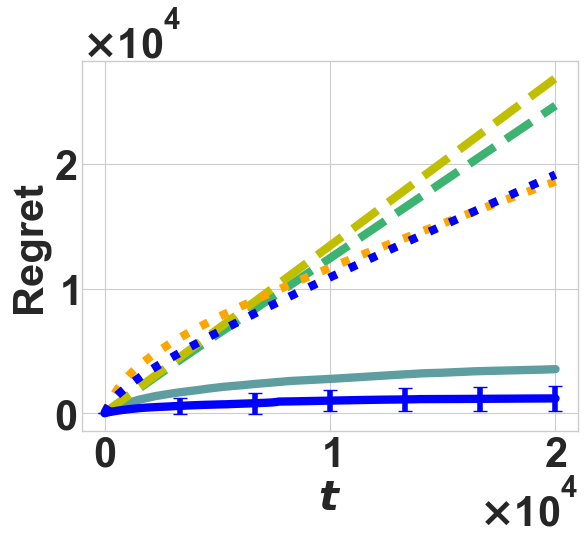}} 
	\subfloat[Star]{\includegraphics[width=\benchmarkwidth]{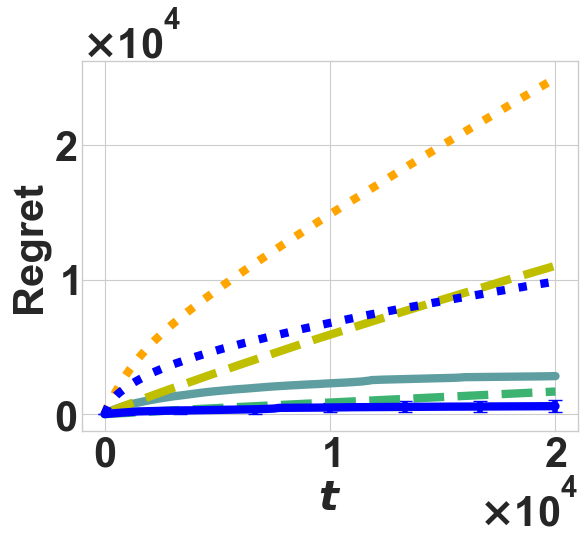}}
	\subfloat[Fully connected]{\includegraphics[width=\benchmarkwidth]{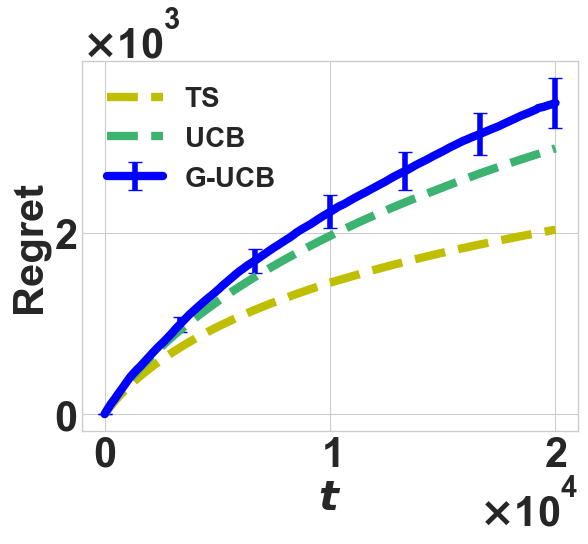}\label{subfig:FC}}
	\caption{
		Regret as a function of time for three benchmark algorithms and the proposed algorithm on six different graph structures. The captions of the subfigures indicate the types of the graph.  Each regret curve is the average over 100 simulations.  For visualization clarity, we plot the error bars showing one standard deviation only for \texttt{G-UCB}.}
	\label{fig:ALL}
\end{figure}

\begin{figure}
	\centering
	\subfloat[Line]{\includegraphics[width=\benchmarkwidth]{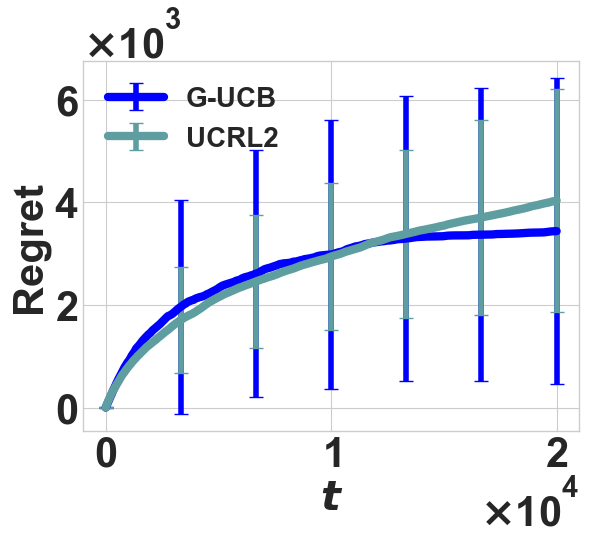} \label{subfig:line}}
	\subfloat[Circle]{\includegraphics[width=\benchmarkwidth]{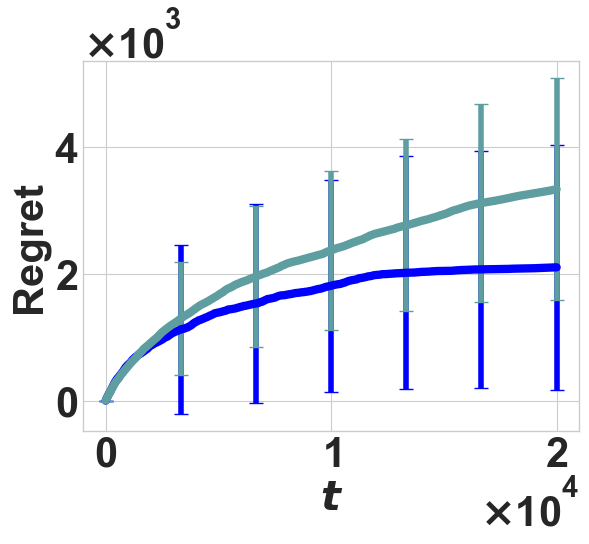} \label{subfig:circle}}  
	\subfloat[Grid]{\includegraphics[width=\benchmarkwidth]{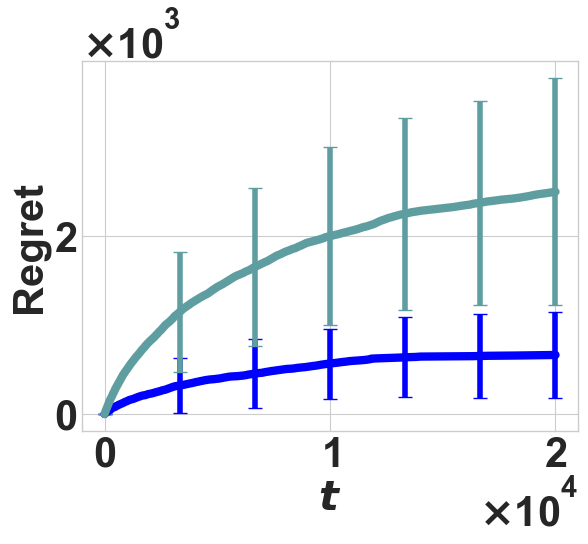}\label{subfig:grid}}
	\\
	\subfloat[Tree]{\includegraphics[width=\benchmarkwidth]{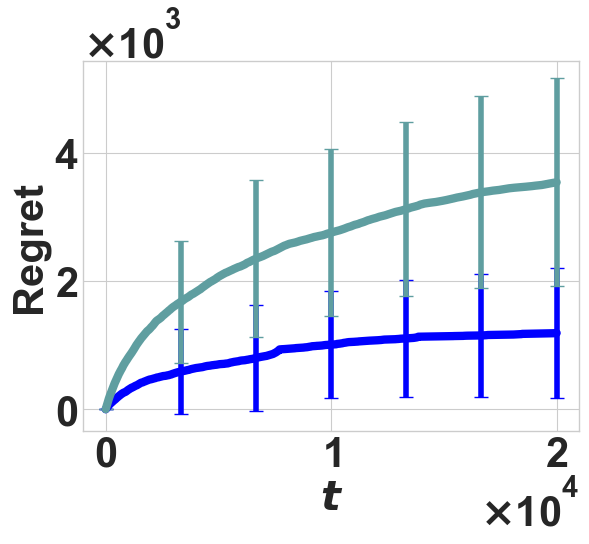}\label{subfig:tree}} 
	\subfloat[Star]{\includegraphics[width=\benchmarkwidth]{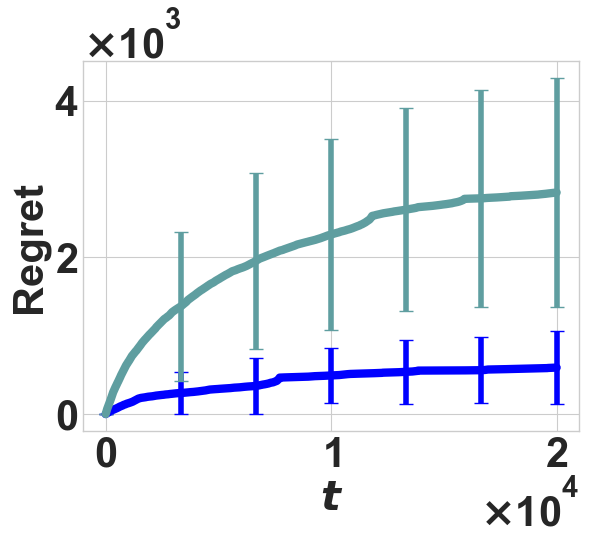}\label{subfig:star}}
	\subfloat[Fully connected]{\includegraphics[width=\benchmarkwidth]{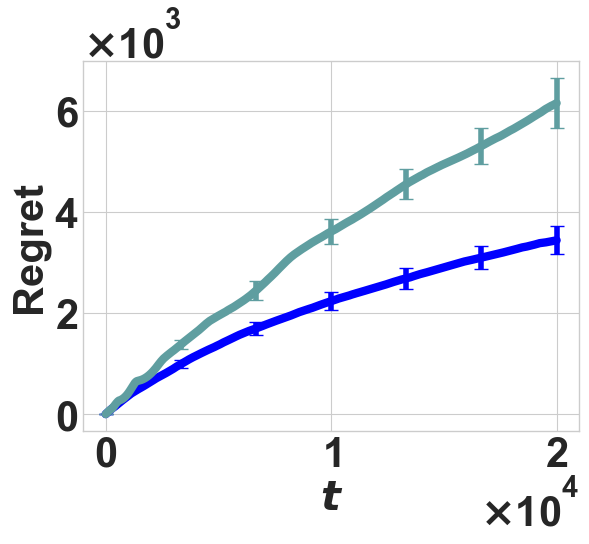}}
	\caption{
		Regret comparison for \texttt{G-UCB} and UCRL2 on six different graph structures. }
	\label{fig:GUCB-vs-UCRL2}
\end{figure}

The experiment on the fully connected graph in Fig. \ref{subfig:FC} demonstrates how \texttt{G-UCB} performs in the original multi-armed bandit problem.
\modified{The result shows that empirically \texttt{G-UCB} performs slightly worse than the classical bandit algorithms, although our regret guarantee for \texttt{G-UCB} is $O(\sqrt{ST\log(T)})$ which matches the regrets of original UCB and Thompson Sampling. This phenomenon arises since the doubling operation in \texttt{G-UCB} makes the algorithm explore more than the classical UCB, which results in worse empirical performances.  We note that \texttt{G-UCB} on a fully-connected graph is equivalent to the phased UCB algorithm, and the reader may consult exercise 7.5 in \cite{lattimore2020bandit} for further information on this phenomenon. }

\subsection{Simulation Efficiency}\label{append:Computation}

This subsection demonstrates the qualitative improvement in computational efficiency of \texttt{G-UCB} over UCRL2 in terms of running time per simulation.  We run the following simulations on a 13" Apple Macbook Air with an M1 chip, 8GB RAM, 256 GB disk space, and macOS Ventura 13.0. The Python version is 3.10.4. No parallelism or any software/hardware acceleration is used. For efficiency, we simulate \texttt{G-UCB} and UCRL2 algorithms under the same graph bandit settings as in Appendix \ref{append:all-benchmark-all-graph}: $|S|=100$ for each of the six different graphs, $T=2\times 10^4$ steps per simulation, with $100$ simulation per algorithm per graph. We record the times for all simulations and present them in Fig. \ref{fig:Computation}. The results show the running time of \texttt{G-UCB} is consistently lower than UCRL2 over different graphs. The median run times per simulation of UCRL2 range from $1.68$ to $11.88$ times the median run time of \texttt{G-UCB} on all graphs. Since offline planning is the computational bottleneck for both algorithms, the significantly less running time of \texttt{G-UCB} per iteration indicates an impressive computational efficiency improvement from using the SP planning algorithm instead of value iteration. This example again demonstrates understanding graph bandit structure brings improvement over directly applying generic RL algorithms.

\begin{figure}[ht]
	\centering
	\subfloat[Run times]{\includegraphics[width=0.45\linewidth]{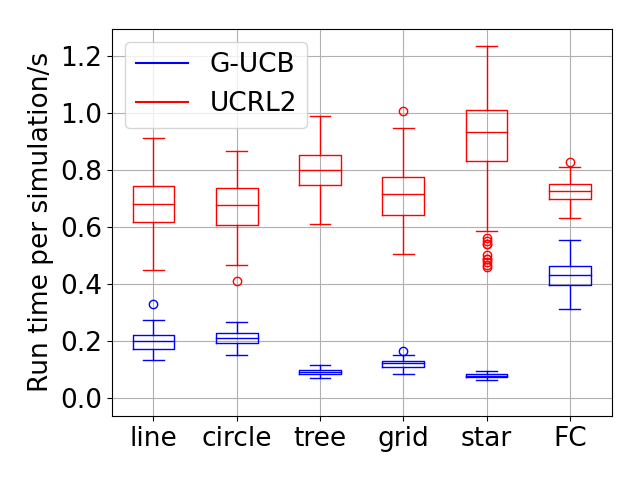}\label{subfig:run-times}}
	\subfloat[Run time ratios]{\includegraphics[width=0.45\linewidth]{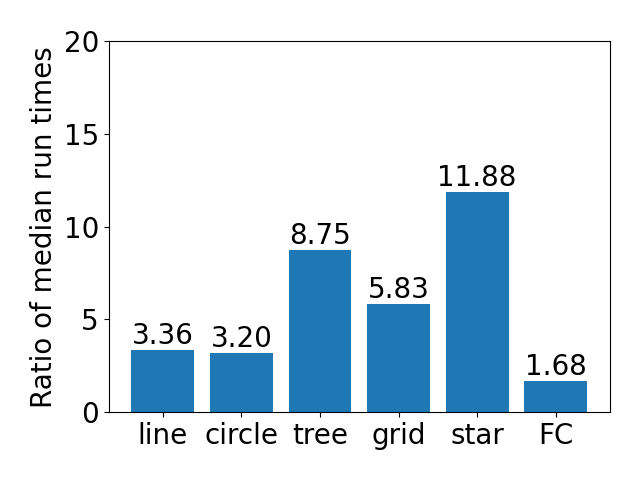}\label{subfig:run-time-ratios}}
	\caption{Computation efficiency of \texttt{G-UCB} and UCRL2 regarding running time per simulation on different graphs. Fig. \ref{subfig:run-times}: The fully connected graph is labeled FC. The run times for UCRL2 are in red, and those for \texttt{G-UCB} are in blue. Fig. \ref{subfig:run-time-ratios}: The ratios values are the median run times of UCRL2 over that of \texttt{G-UCB}.} 
	\label{fig:Computation}
\end{figure}

\subsection{Performance Improvement Brought by our UCB}\label{append:improvement-by-UCB}

In this subsection, we show that our different UCB definition from \cite{Jaksch2010a} leads to a prominent improvement of \texttt{G-UCB} over UCRL2 in empirical performance. The UCB in \cite{Jaksch2010a} is in defined as \begin{equation*}
	\begin{aligned}
		\tilde{U}_{m-1}(s) &= \bar{\mu}_{m-1}(s)+\sqrt{\frac{7\log(SA t_m/\delta)}{2n_{m-1}(s)}},
	\end{aligned}
\end{equation*}  where $S, ~A$ are the number of states and actions, and $\delta\in(0,1]$ is a confidence parameter. The UCB definition above has an explicit dependence on $S$ and $A$, which is absent in our UCB(Eq. \eqref{eq:UCB}). This dependence on $S$ and $A$ arises in the confidence radius due to the analysis under the general RL setting, but our analysis under the graph bandit setting does not hint at such dependence. Directly applying this UCB to graph bandit should result in more exploration and higher regret than our UCB. This intuition is confirmed in the experiment below.

We create a modified \texttt{G-UCB} by replacing the $U_{m-1}$ values in line 6 of \texttt{G-UCB} with the UCB $\Tilde{U}_{m-1}$ above and compare the resulting regret with the original \texttt{G-UCB}. The experiments are conducted under the same graph bandit environment setting as in Appendix \ref{append:all-benchmark-all-graph}. The results in Fig. \ref{fig:UCB_comparison} show that the regret when using our UCB is consistently lower than the UCB in UCRL2, with a margin of more than two standard deviations on all graphs. This impressive improvement indicates that directly using the UCB in \cite{Jaksch2010a} on graph bandit would result in more exploration and higher regret, but with a good understanding of the structures of graph bandit, we have shown that a more appropriate UCB can be used to achieve much better empirical learning results. 

\begin{figure}[ht]
	\centering
	\subfloat[Line]{\includegraphics[width=0.33\linewidth]{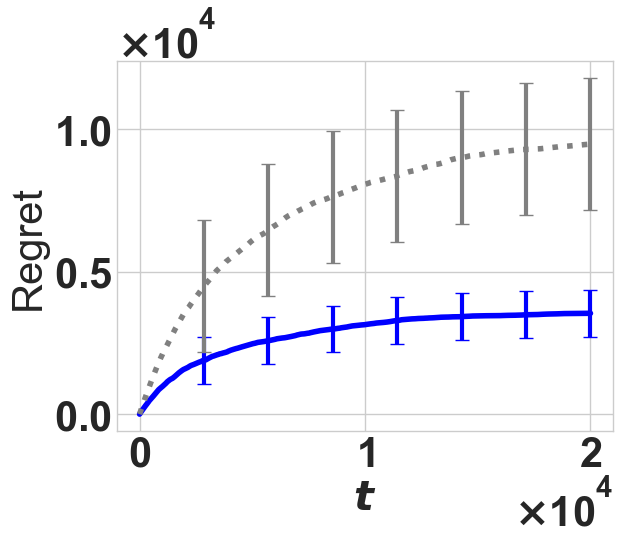}}
	\subfloat[Circle]{\includegraphics[width=\ucbcomparewidth]{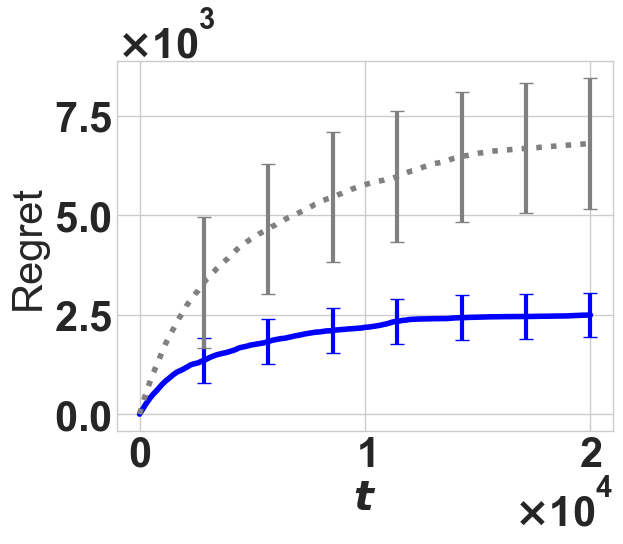}}  
	\subfloat[Grid]{\includegraphics[width=\ucbcomparewidth]{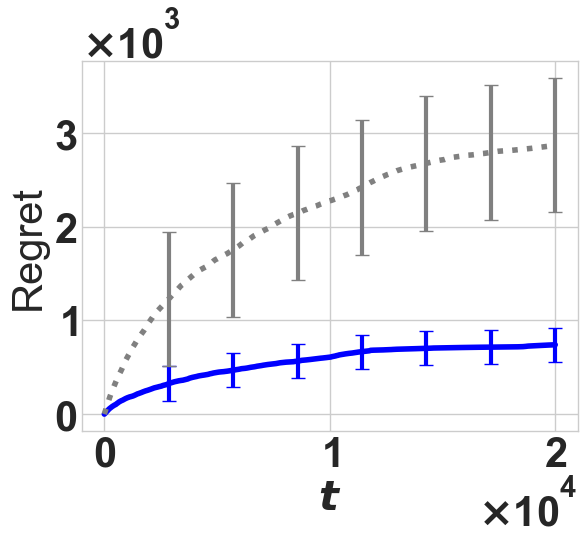}}
	\\
	\subfloat[Tree]{\includegraphics[width=\ucbcomparewidth]{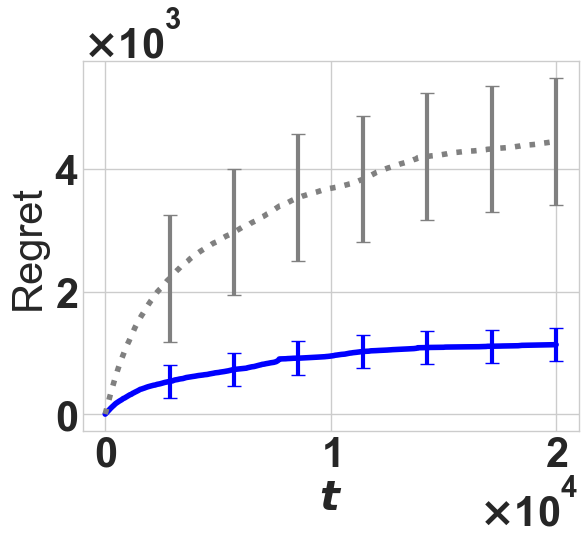}} 
	\subfloat[Star]{\includegraphics[width=\ucbcomparewidth]{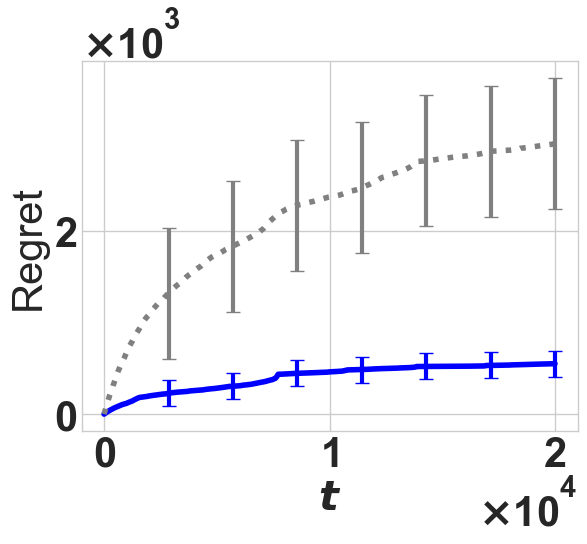}}
	\subfloat[Fully connected]{\includegraphics[width=\ucbcomparewidth]{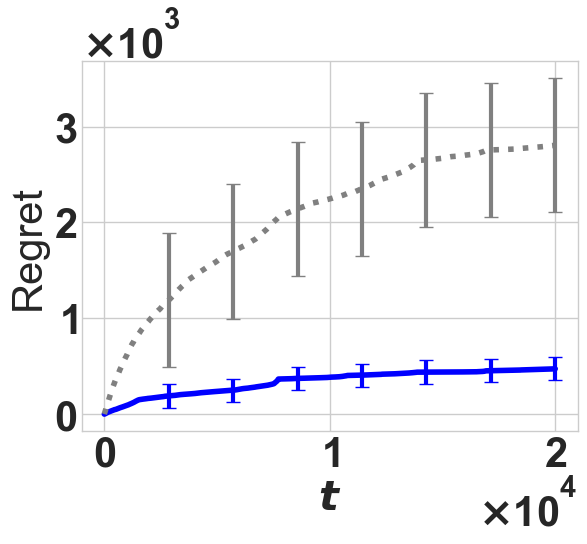}}
	\caption{
		Regret of \texttt{G-UCB} with two different definitions of UCB.  Gray dotted curves: regrets when using the UCB defined in the UCRL2 algorithm\cite{Jaksch2010a}. Blue solid curves: regrets of the original \texttt{G-UCB}. Using our UCB in Eq. \eqref{eq:UCB}  shows a clear advantage over the UCB in UCRL2\cite{Jaksch2010a}.
	}
	\label{fig:UCB_comparison}
\end{figure}

\subsection{Comparison with the doubling scheme of UCRL2}\label{append:Doubling-Comparison}

In the following experiments, we aim to show if the small difference in the doubling schemes of \texttt{G-UCB} and UCRL2 contributes to the difference in their learning regret. The two doubling schemes are only different in the terminal conditions of one episode: in \texttt{G-UCB}, the episode ends when the samples of the destination node are doubled; in comparison, there is not a destination node in UCRL2, and the episode terminates as long as the samples of \textit{some} node in the graph is doubled. We note that these two doubling schemes coincide if the time horizon is large enough and the optimality of the policies is ensured; therefore, they should result in similar regrets.

In the experiments, the \texttt{G-UCB} algorithm is used as the control group, represented by blue solid curves. The comparison group is a modification of \texttt{G-UCB} by replacing its doubling scheme with the doubling scheme of UCRL2, which corresponds to the gray dotted curves. We run the learning experiments for these two groups using the same default experiment setting specified in Appendix \ref{append:all-benchmark-all-graph}. Results from Fig. \ref{fig:Doubling_comparison} show that the differences in regrets of the two groups are all within one standard deviation except for the experiment on the tree graph, where the doubling scheme of \texttt{G-UCB} shows a slight advantage over UCRL2. Overall, no decisive evidence shows that different doubling schemes lead to significant differences in learning regrets. The result implies the main contributor to the improvement of \texttt{G-UCB} over UCRL2 in empirical learning regret is the more appropriate definition of UCB derived using the structure of graph bandit rather than the doubling scheme.

\begin{figure}[ht]
	\centering
	\subfloat[Line]{\includegraphics[width=0.33\linewidth]{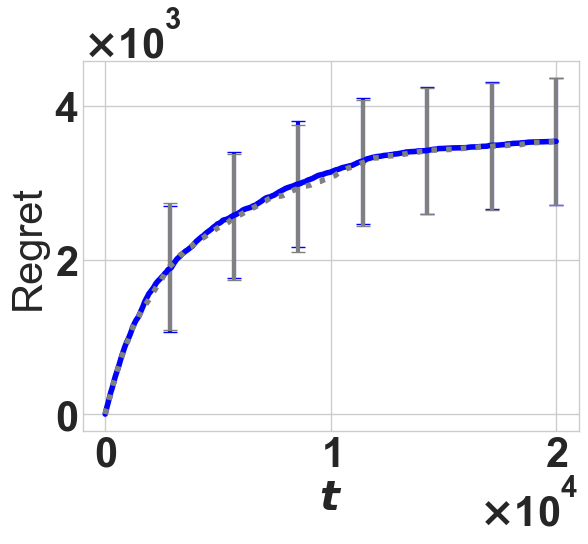}}
	\subfloat[Circle]{\includegraphics[width=\ucbcomparewidth]{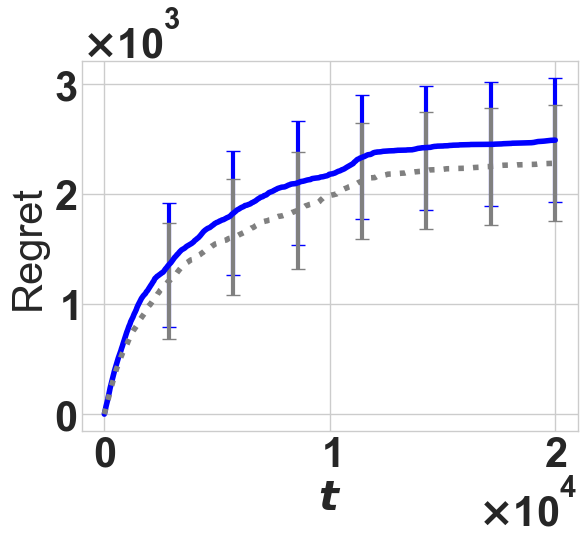}}  
	\subfloat[Grid]{\includegraphics[width=\ucbcomparewidth]{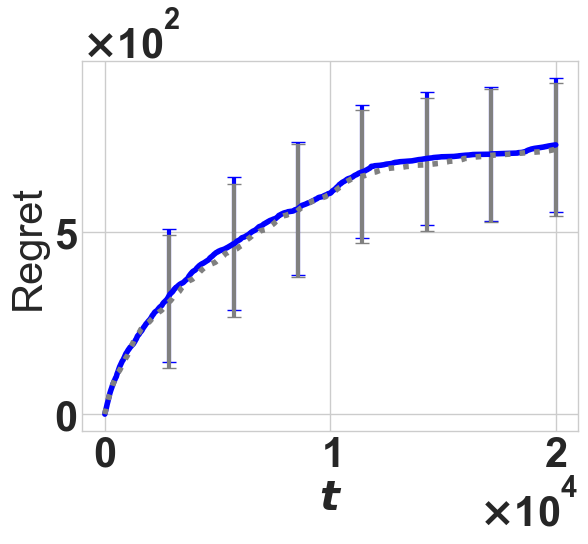}}
	\\
	\subfloat[Tree]{\includegraphics[width=\ucbcomparewidth]{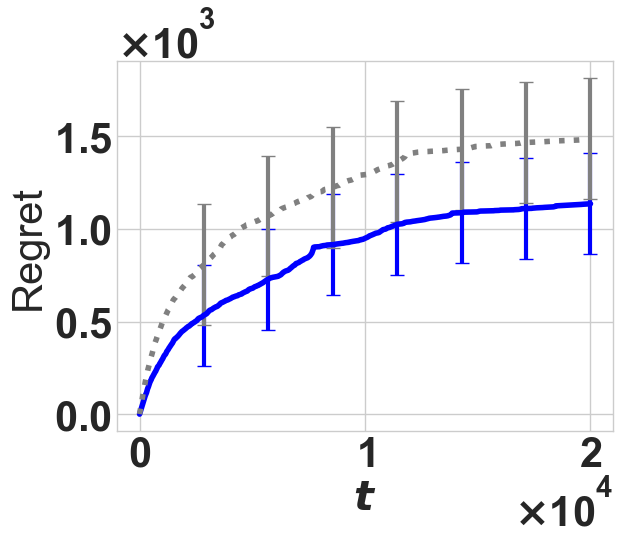}} 
	\subfloat[Star]{\includegraphics[width=\ucbcomparewidth]{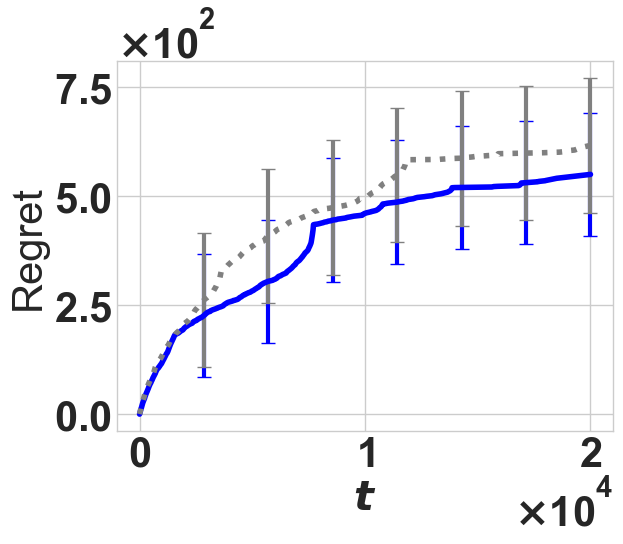}}
	\subfloat[Fully connected]{\includegraphics[width=\ucbcomparewidth]{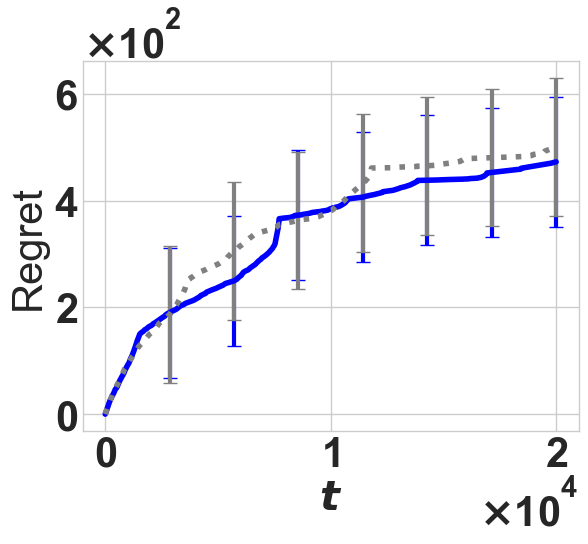}}
	\caption{
		Regret under different doubling schemes. The gray dotted curve corresponds to the doubling scheme from the UCRL2 algorithm\cite{Jaksch2010a}; the blue solid curve corresponds to the doubling scheme from \texttt{G-UCB} algorithm. The error bars plot one standard deviation across 100 trials. All experiments use the UCB and offline planning as in the \texttt{G-UCB} algorithm. There is no decisive difference between the doubling schemes in \texttt{G-UCB} and UCRL2.
	}
	\label{fig:Doubling_comparison}
\end{figure}

\subsection{Sensitivity to Environment Parameters}
In the following experiments, we study the effect of $|S|,D,$ and $\Delta$ on the regret of our algorithm. Unless specified otherwise, the total time $T$ is $1000$ in these experiments, the mean rewards are initialized in the same way as the previous experiments, and the regret is averaged over 100 simulations. 
\textbf{Dependence on $|S|$.}
We experiment with the effect of the number of nodes $|S|$ by running the experiments on the star graphs, with a fix diameter $D$ but an increasing number of ``branches''. Fig. \ref{subfig:RCurves|S|} shows an increasing trend of regret as $|S|$ increases. The regret at time $t=1000$ is shown in Fig. \ref{subfig:Rvs|S|}. With the total time fixed, we can see regret is indeed sublinear in $|S|$, which agrees with the bound in Theorem \ref{thm:G-UCB} that is proportional to $\sqrt{|S|}$. In particular, the regret at time $T$ flattens when $|S|$ is approximately $T$.

\begin{figure}[ht]
	\centering
	\subfloat[Regret curves]{\includegraphics[width=0.45\linewidth]{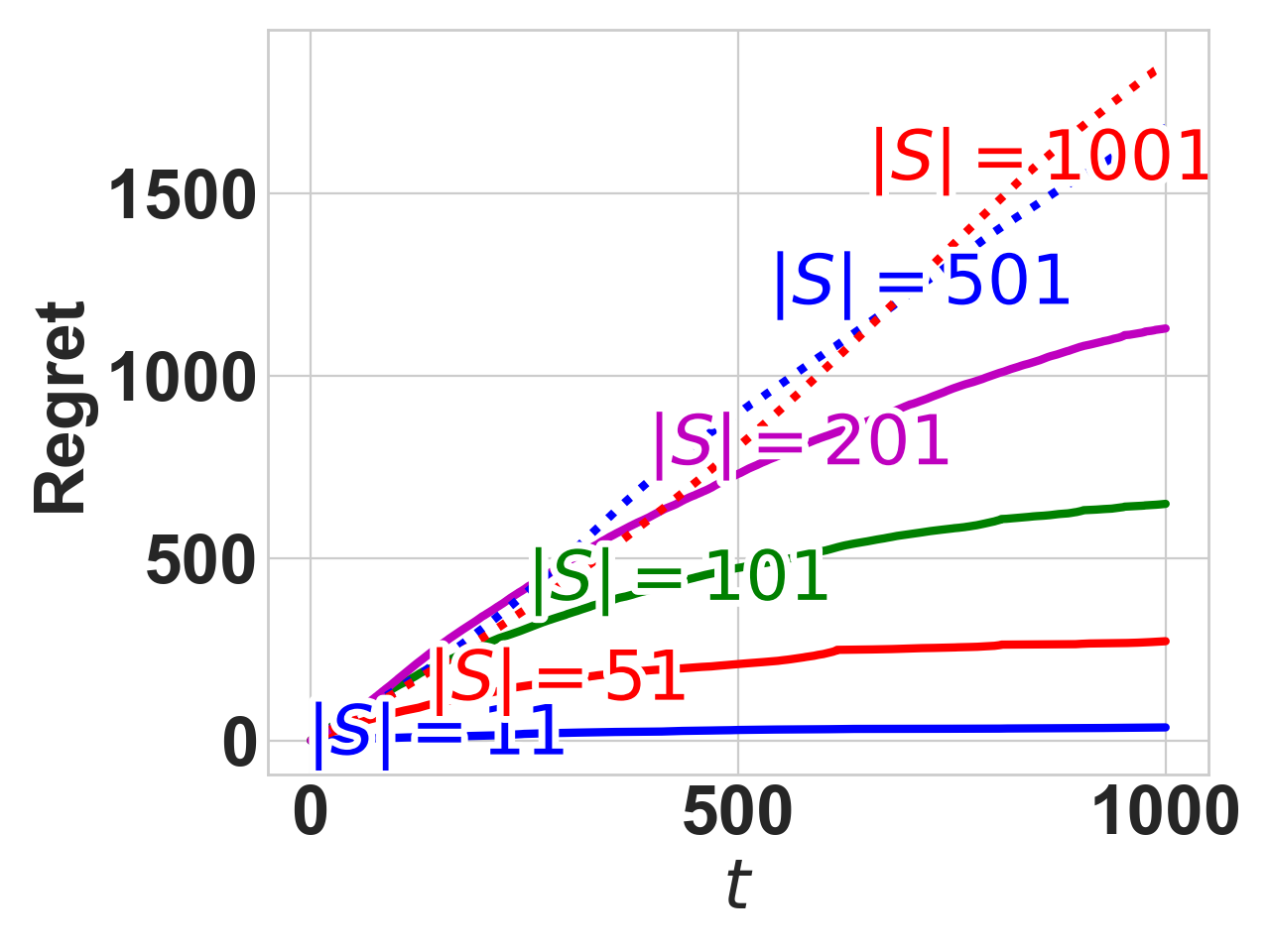}\label{subfig:RCurves|S|}}
	\subfloat[Regrets at $t$=1000. ]{\includegraphics[width=0.45\linewidth]{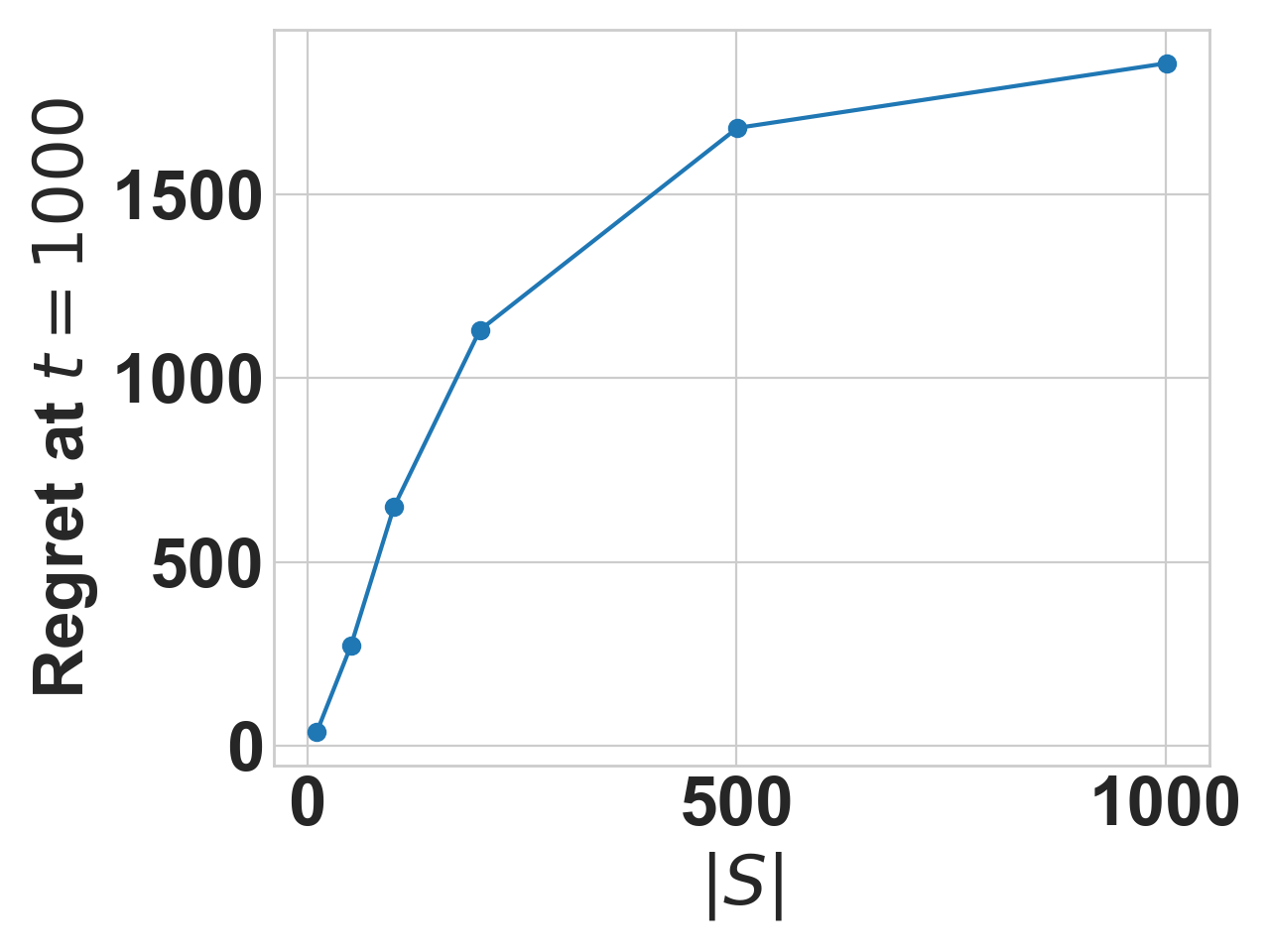}\label{subfig:Rvs|S|}}
	\caption{The algorithm's dependence on  $|S|$.}
	\label{fig:num_nodes_test}
	
\end{figure}

\textbf{Dependence on $D$.} We also experiment with the influence of graph diameter $D$ on our algorithm. We fix the number of nodes $|S|=50$ and systematically increase the diameter from $D=2$ to $D=|S|$(see Appendix \ref{appendix:IncreasingD} for the detailed information). Fig. \ref{fig:D_test} shows that overall the regret grows linearly in $D$, which agrees with the regret bound in Theorem \ref{thm:G-UCB}. The regret stops growing when $D$ is very close to $|S|$, as the graph has almost become a line graph.

\begin{figure}[ht]
	\centering
	\subfloat[Regret curves]{\includegraphics[width=0.45\linewidth]{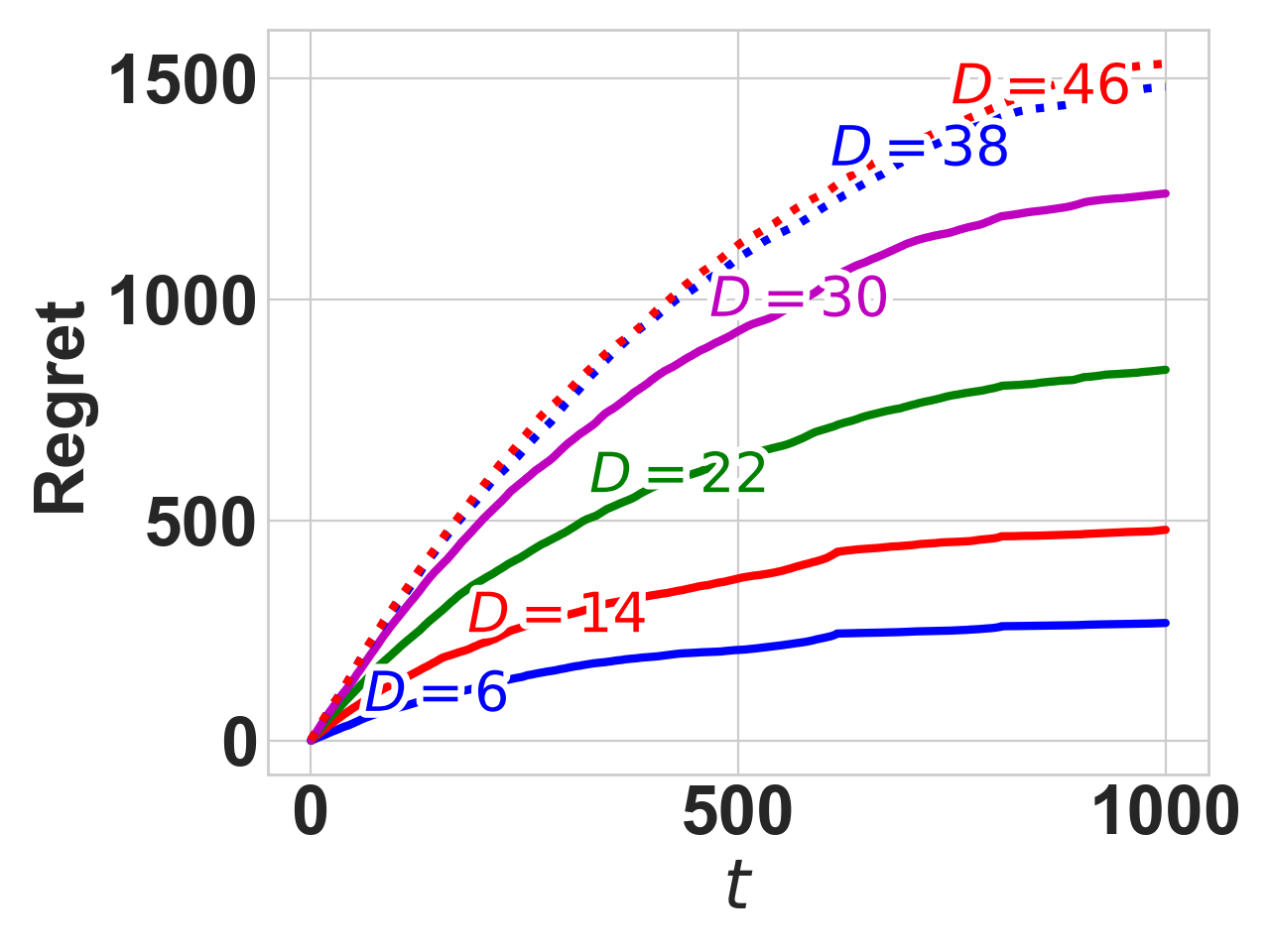}\label{subfig:D_test}}
	\subfloat[Regrets at $t$=1000. ]{\includegraphics[width=0.45\linewidth]{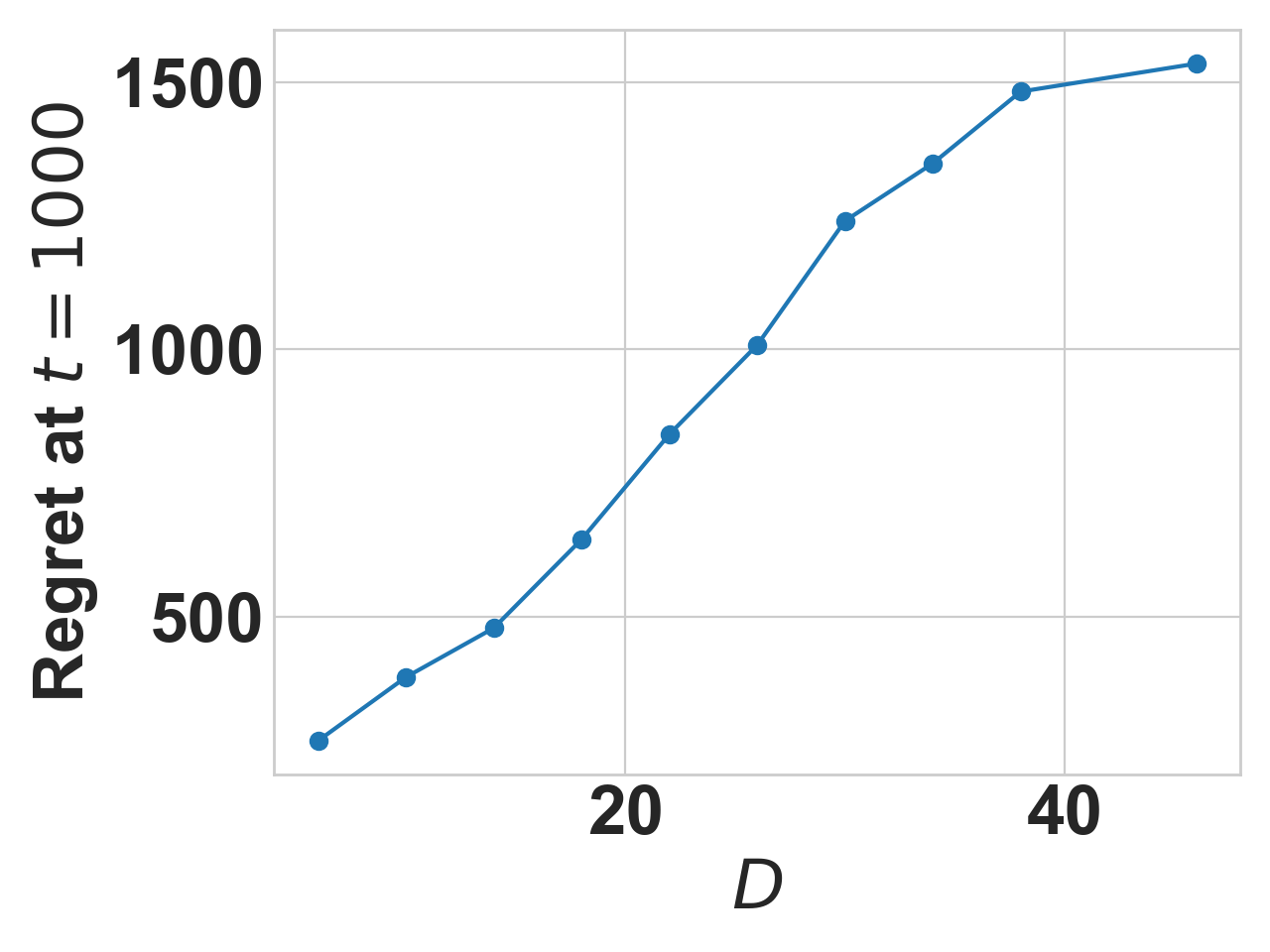}\label{subfig:RvsD}}
	\caption{The algorithm's dependence on $D$.}
	\label{fig:D_test}
\end{figure}
\textbf{Dependence on $\Delta$.}
We evaluated the effect of the gap $\Delta$ on a line graph of $|S|=10$ nodes so that $|S|$ and $D$ are fixed. The first node in the graph has a mean of $9.5$, the last node has a mean of $9.5-\Delta$, and all remaining nodes yield zero rewards. The empirical regrets from our simulations are illustrated in Fig. \ref{fig:delta_test}. The regrets for $\Delta = 0.1,0.01,10^{-3},10^{-4}$ are collectively labeled as $\Delta\leq 0.1$.

Like MAB, the graph bandit becomes more difficult to learn as $\Delta$ decreases. We can see that the regret increases as $\Delta$ decreases, which agrees with the $1/\Delta$ dependence in our instance-dependent bound in Theorem \ref{thm:UCB-DEST-DELTA}. But interestingly, the regret stops increasing as $\Delta$ becomes less than $0.1$,  which may imply that the regret has reached the minimax bound stated in Theorem \ref{thm:G-UCB}. 
\begin{figure}[ht]
	\centering
	\includegraphics[width=0.5\linewidth]{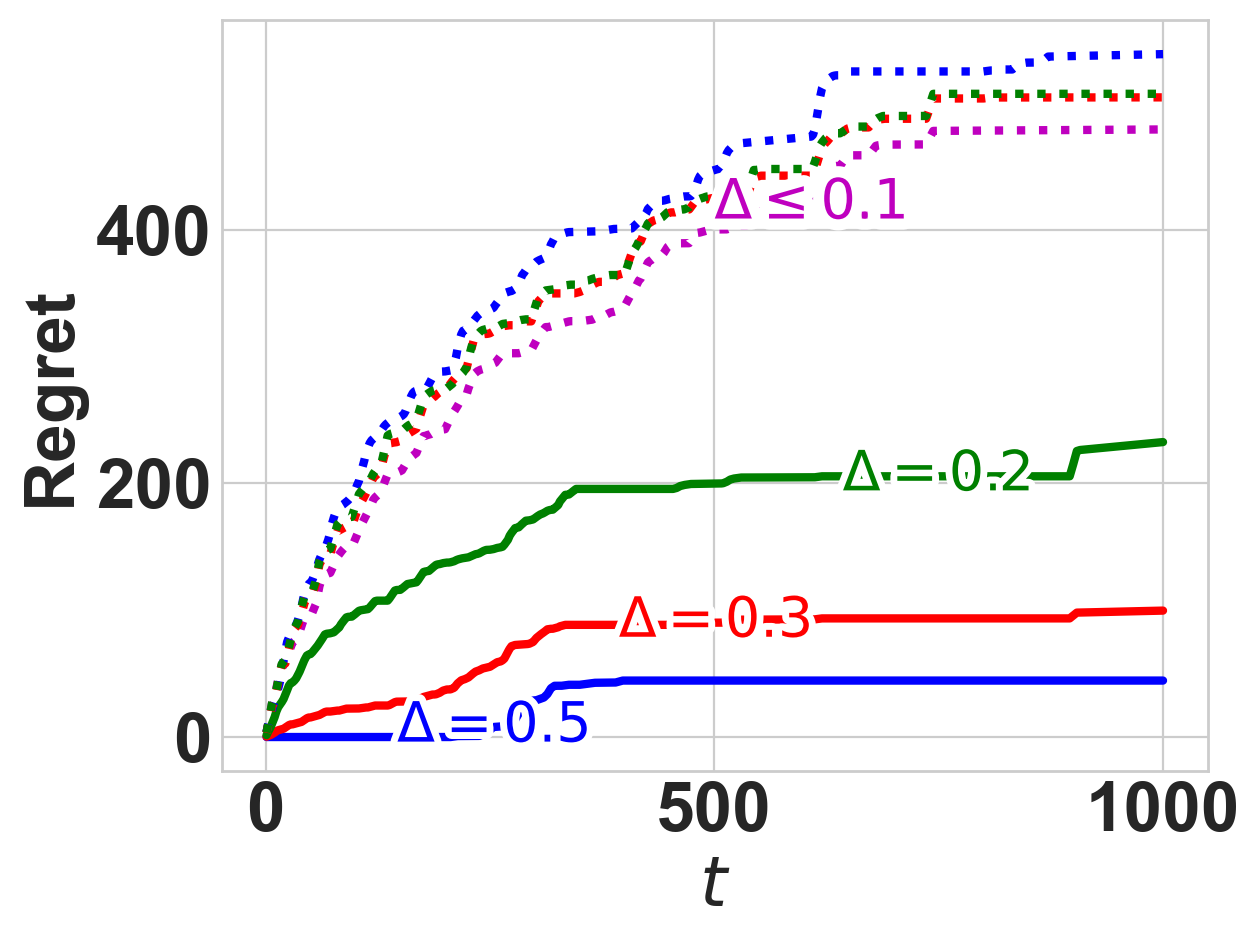}
	\caption{The dependence on the gap $\Delta$.}
	\label{fig:delta_test}
\end{figure}

\subsection{Variation in SP Offline Planning.} 

\begin{figure}[ht]
	\centering
	\subfloat[Line]{\includegraphics[width=0.30\linewidth]{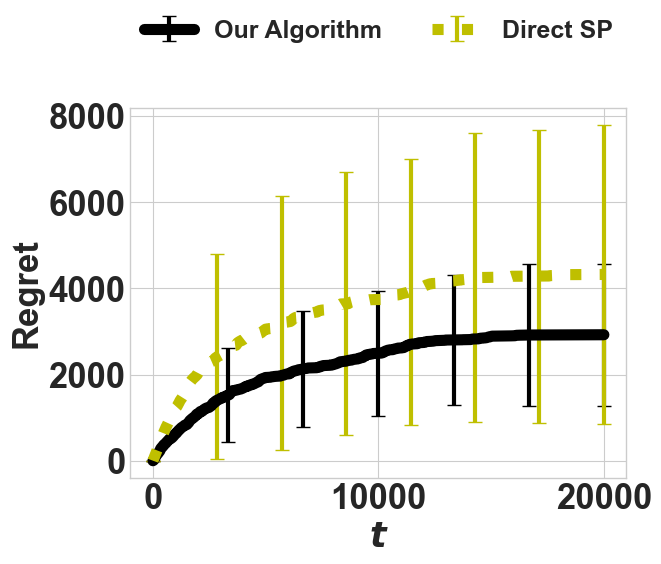} \label{subfig:line_direct}}
	\subfloat[Grid]{\includegraphics[width=0.30\linewidth]{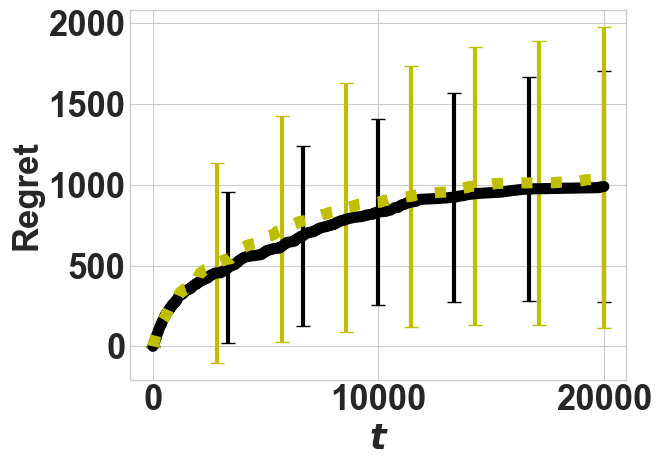}\label{subfig:grid_direct}}
	\\
	\subfloat[Circle]{\includegraphics[width=0.30\linewidth]{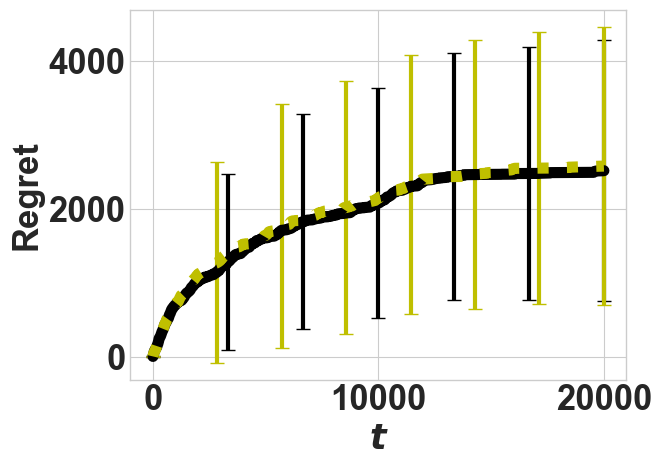} \label{subfig:circle_direct}}  
	\subfloat[Tree]{\includegraphics[width=0.30\linewidth]{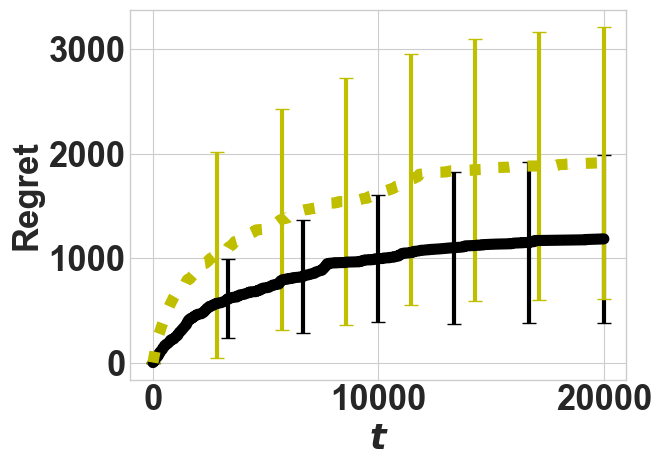}\label{subfig:tree_direct}} 
	\subfloat[Star]{\includegraphics[width=0.30\linewidth]{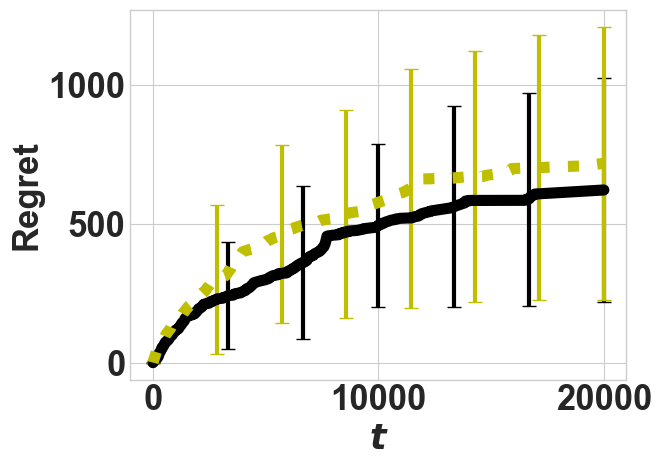}\label{subfig:star_direct}}
	\caption{
		Regret as a function of time for two different ways of reaching $dest_m$. The error bars show one standard deviation.}
	\label{fig:directSP}
\end{figure}

We note that the result of Theorem \ref{thm:G-UCB} still holds if the agent follows the path with the shortest length to $dest_m$ instead of executing the optimal $\pi_m$ policy. Although the theoretical guarantees are unchanged, there may be differences in empirical performance, and we present the corresponding simulation results in Fig. \ref{fig:directSP}. The experiments use the same configurations as in the first set of experiments in Section 5.1 of the main paper, with the fully-connected graph omitted for simplicity. The black solid curves correspond to the regrets of our original algorithm, and the yellow dotted curves(labeled as `Direct SP') correspond to replacing the execution of $\pi_m$ with following the path of shortest length to $dest_m$. 

The results show that the performance of the original algorithm and the modification are overall similar. The difference in average regrets is smaller than one standard deviation. Nevertheless, our original algorithm does perform better on graphs with a higher $D/|S|$ ratio, such as line, grid, and circle graphs. This result suggests that while the doubling operation is the predominant mechanism that ensures the $\sqrt{T}$ regret, following $\pi_m$ is beneficial empirically on graphs where $D$ is large such that the cost of destination switch is high.

\section{The method of increasing $D$ while keeping $|S|$ unchanged}\label{appendix:IncreasingD}

An illustration of how we increase $D$ while keeping $|S|$ unchanged in the $D$-dependence experiment is shown in Fig. \ref{fig:IncreasingD}. The code can be found in our supplementary materials.
\begin{figure}[ht]
	\centering
	\includegraphics[width=0.6\linewidth]{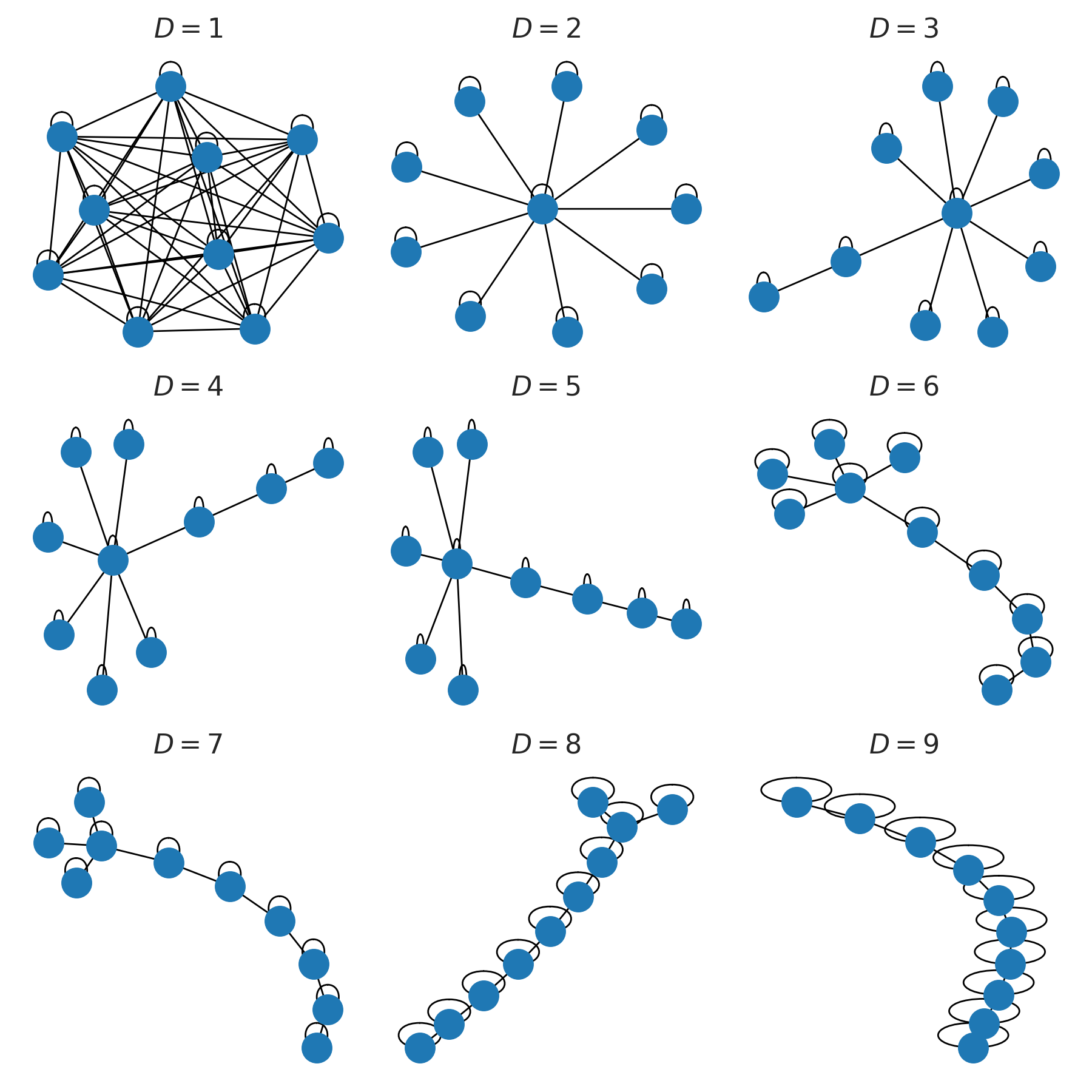}
	\caption{An illustration increasing $D$ while keeping $|S|=10$ unchanged.}
	\label{fig:IncreasingD}
\end{figure}

\fi
\end{document}